\documentclass[lettersize,journal]{IEEEtran}
\usepackage{amsmath,amsfonts}
\usepackage{algorithmic}
\usepackage{algorithm}
\usepackage{array}
\usepackage[caption=false,font=normalsize,labelfont=sf,textfont=sf]{subfig}
\usepackage{textcomp}
\usepackage{stfloats}
\usepackage{url}
\usepackage{verbatim}
\usepackage{graphicx}
\usepackage{cite}
\usepackage{multirow}
\usepackage{booktabs}
\usepackage{amsthm}
\usepackage{caption}

\usepackage[normalem]{ulem}
\useunder{\uline}{\ul}{}
\newtheorem{proposition}{Proposition}[section]
\begin{document}

%\title{Adversarial Analysis of Signed Graphs\\ with Balance Theory}
\title{Adversarial Robustness of Link Sign Prediction in Signed Graphs}
% \title{A Sample Article Using IEEEtran.cls\\ for IEEE Journals and Transactions}

\author{Jialong Zhou, Xing Ai, Yuni Lai, Tomasz Michalak, Gaolei Li, Jianhua Li,\\ Di Tang, Xingxing Zhang, Mengpei Yang, Kai Zhou
        % <-this % stops a space
\thanks{Jialong Zhou, Xing Ai, Yuni Lai, and Kai Zhou are with the Department of Computing, The Hong Kong Polytechnic University, Hong Kong, China (e-mail: jialong.zhou@connect.polyu.hk, xing96.ai@connect.polyu.hk, csylai@comp.polyu.edu.hk, kaizhou@polyu.edu.hk).}% <-this % stops a space
\thanks{Tomasz Michalak is with the University of Warsaw \& Ideas NCBR, Warsaw, Poland (e-mail: tpm@mimuw.edu.pl).}
\thanks{Gaolei Li and Jianhua Li are with the School of Electronic Information and Electrical Engineering, Shanghai Jiao Tong University, Shanghai, China (e-mail: gaolei\_li@sjtu.edu.cn, lijh888@sjtu.edu.cn).}
\thanks{Di Tang and Xingxing Zhang  are with the Shanghai CESI Technology Co., Ltd. Mengpei Yang is with the China Electronics Standardization Institute (e-mail: ditonytang@hotmail.com, zhangxx@cesi.cn, yangmp@cesi.cn).}
}



\maketitle

\begin{abstract}
Signed graphs serve as fundamental data structures for representing positive and negative relationships in social networks, with signed graph neural networks (SGNNs) emerging as the primary tool for their analysis. Our investigation reveals that balance theory, while essential for modeling signed relationships in SGNNs, inadvertently introduces exploitable vulnerabilities to black-box attacks. To showcase this, we propose balance-attack, a novel adversarial strategy specifically designed to compromise graph balance degree, and develop an efficient heuristic algorithm to solve the associated NP-hard optimization problem. While existing approaches attempt to restore attacked graphs through balance learning techniques, they face a critical challenge we term ``Irreversibility of Balance-related Information," as restored edges fail to align with original attack targets. To address this limitation, we introduce Balance Augmented-Signed Graph Contrastive Learning (BA-SGCL), an innovative framework that combines contrastive learning with balance augmentation techniques to achieve robust graph representations. By maintaining high balance degree in the latent space, BA-SGCL not only effectively circumvents the irreversibility challenge but also significantly enhances model resilience. Extensive experiments across multiple SGNN architectures and real-world datasets demonstrate both the effectiveness of our proposed balance-attack and the superior robustness of BA-SGCL, advancing the security and reliability of signed graph analysis in social networks. Datasets and codes of the proposed framework are at the github repository https://anonymous.4open.science/r/BA-SGCL-submit-DF41/.
\end{abstract}


\section{Introduction}
\IEEEPARstart{H}{uman} relationships encompass a broad spectrum of connections, from positive interactions like trust and support to negative associations like distrust and conflict. Signed graphs have emerged as a powerful tool to represent these dual-natured relationships by assigning corresponding signs $(+/-)$ to edges. A fundamental task in signed graph analysis is \textit{link sign prediction}~\cite{song2015link,tang2016survey,leskovec2010predicting}, which aims to predict signs of remaining edges based on partially observed graph information. Signed Graph Neural Networks (SGNNs) \cite{zhang2024signed}, such as SGCN \cite{derr2018signed} and SDGNN \cite{huang2021sdgnn}, have emerged as the dominant models for this task. A distinctive feature of SGNNs is their utilization of \textit{balance theory} \cite{davis1967clustering,cartwright1956structural,kirkley2019balance}, a fundamental social science theory that governs the distribution of signs in signed networks.

Despite their effectiveness, SGNNs have demonstrated vulnerability to adversarial attacks. In real-world applications such as bitcoin trading platforms and e-commerce sites, malicious users can manipulate these signed networks by altering their edge signs—for instance, by providing false ratings or deliberately misrepresenting relationship signs. Such manipulations, even when affecting only a small portion of these signs, can significantly degrade SGNN performance and potentially compromise the integrity of critical systems.

In our prior work \cite{zhou2024black}, we systematically studied this vulnerability by introducing \textit{balance-attack}, a novel black-box poisoning attack. The core idea is to poison the graph's training data by targeting the fundamental mechanism of SGNNs: balance theory. Previous research has demonstrated that SGNNs cannot learn accurate node representations from unbalanced triangles \cite{zhang2023rsgnn}. By strategically manipulating edge signs to reduce the graph's balance degree, our method exploits this inherent weakness of SGNNs in learning from unbalanced structures. While this presents an NP-hard optimization problem \cite{diao2020approximation}, we propose an efficient heuristic algorithm that effectively compromises the performance of existing SGNN models.

The success of our balance-attack and other methods like FlipAttack \cite{zhu2024towards} highlights a critical vulnerability in SGNNs. This is a pressing issue as they are increasingly used in security-sensitive areas like Bitcoin trust prediction \cite{godziszewski2021attacking,grassia2022wsgat,huang2022pole}. However, while our previous work focused on exposing this vulnerability, a robust defense remains an open and urgent problem. The only existing method, RSGNN \cite{zhang2023rsgnn}, is designed for random noise and, as our experiments show, fails against targeted adversarial attacks. This significant research gap motivates our work: to design the first truly adversarially robust SGNN.

\textbf{Adaptation from previous defense}.\quad Our initial effort was to adapt \textit{structural learning}, a proven method for making GNNs robust on \textit{unsigned} graphs \cite{jin2020graph,xu2021speedup,li2024gslb,zhu2023focusedcleaner}. The essential idea is to refine the poisoned graph structure using some proper guidance as the learning objective. To adapt this to signed graphs, we utilize the balance degree as this objective, since attacks tend to decrease it. Hence, we refer to this customization of structural learning applied to signed graphs as \textit{balance learning}. Unfortunately, the balance learning approach fails to exhibit good robustness in the face of attacks. Our further investigation reveals that while balance learning can effectively restore the balance degree, it fails to recover the distribution of signs -- we term this challenge as \textit{``Irreversibility of Balance-related Information''} (detailed later). Thus, the failure of this intuitive adaptation underscores the need for a more sophisticated, dedicated defense for SGNNs.

\textbf{Our solution}.\quad We propose a novel robust SGNN model, \textit{Balance Augmented-Signed Graph Contrastive Learning} (BA-SGCL), which builds upon the Graph Contrastive Learning (GCL) framework \cite{sun2019infograph,qiu2020gcc,you2020graph} to \textit{indirectly} address the Irreversibility of Balance-related Information challenge. Our core idea is to contrast a positive view, which is an augmented version of the graph with an enhanced balance degree, against the original input graph as the negative view. To generate this positive view, we utilize the balance degree as a guiding factor to shape the Bernoulli probability matrix within a learnable augmenter \cite{hobert2011data}. By maximizing the mutual information between these two views \cite{abdal2019image2stylegan,liu2019latent,zhang2021we}, our model learns node embeddings that are implicitly characterized by a high balance degree, bypassing the need for direct graph recovery. In conjunction with a supervised objective that maximizes the mutual information between embeddings and labels, our approach effectively defends against attacks and improves prediction accuracy.

The major contributions of our paper are as follows:
\begin{itemize}
    \item We provide a theoretical analysis of how balance-oriented attacks impact SGNNs from an information-theoretic perspective, offering insights into the vulnerability we previously identified.
    \item We identify and formalize a fundamental challenge in defending SGNNs, the ``Irreversibility of Balance-related Information," which helps explain why conventional defense paradigms struggle in this context.
    \item Building upon this analysis, we propose a novel robust model, BA-SGCL, and provide theoretical justification for its design principles.
    \item Our extensive experiments demonstrate that the proposed BA-SGCL model significantly outperforms existing baselines under various adversarial attacks on signed graphs.
\end{itemize}

\section{Related Work}

Research in adversarial attacks has been extensively explored across various machine learning models \cite{lai2023towards,sun2022adversarial,zhu2019robust}. Unlike naturally occurring outliers, adversarial examples are intentionally crafted with subtle perturbations to deceive machine learning models. GNNs have been shown to be particularly susceptible to these small adversarial perturbations. As a result, numerous studies have focused on adversarial attacks for graph learning tasks. For instance, Bojchevski et al. \cite{bojchevski2019adversarial} proposed a poisoning attack on unsupervised node embedding methods, leveraging perturbation theory to maximize the loss after training. In another line of work, Zugner et al. \cite{zugner_adversarial_2019} tackled the bi-level optimization problem inherent in training-time attacks by employing meta-gradients.

In the context of signed graphs, research on adversarial attacks is more limited but has seen growing interest \cite{godziszewski2021attacking,zhou2024black,zhu2024towards,bu2024uncovering}. Early work by Godziszewski et al. \cite{godziszewski2021attacking} introduced an attack on sign prediction where the goal is to conceal target link signs by removing non-target link signs. More recently, balance-attack \cite{zhou2024black} demonstrated effective black-box attacks on SGNNs by decreasing the graph's balance degree. Unfortunately, SGNNs currently lack robust defense mechanisms against such attacks. While RSGNN \cite{zhang2023rsgnn} is a notable model designed for robustness, it primarily excels at handling random noise and shows limited efficacy against targeted adversarial attacks.

To address these challenges, our approach builds upon Graph Contrastive Learning (GCL) \cite{sun2019infograph,qiu2020gcc,you2020graph,ai2024graph}. GCL aims to learn invariant and generalized node representations by maximizing the correspondence between different augmented views of a graph. A common practice in GCL is to use graph augmentation to generate multiple views for contrastive pairing \cite{bielak2022graph,hassani2020contrastive,thakoor2021bootstrapped}, which helps the model capture essential graph properties. For signed graphs specifically, SGCL \cite{shu2021sgcl} adapts contrastive learning by creating augmented views that preserve signed structures. UGCL \cite{ko2023universal} further improves stability with Laplacian perturbation. However, these existing GCL-based methods for signed graphs are designed for clean graph representations and do not explicitly model or defend against adversarial attacks.

\section{Preliminaries}

\subsection{Notations}

We define a signed directed graph as $\mathcal{G}=(\mathcal{V},\mathcal{E}^+,\mathcal{E}^-)$, where $\mathcal{V}=\{v_1, \dots, v_n\}$ is the set of $n$ nodes. The sets of positive and negative edges are denoted by $\mathcal{E}^+ \subseteq \mathcal{V} \times \mathcal{V}$ and $\mathcal{E}^- \subseteq \mathcal{V} \times \mathcal{V}$ respectively, with the constraint $\mathcal{E}^+ \cap \mathcal{E}^- = \emptyset$. The sign of an edge $e_{ij}$ from node $v_i$ to $v_j$ is denoted by $s_{ij} \in \{+1, -1\}$.

The graph's structure and signs are represented by an adjacency matrix $\mathbf{A} \in \{-1, 0, 1\}^{n \times n}$, where $A_{ij} = s_{ij}$ if an edge exists from $v_i$ to $v_j$, and $A_{ij} = 0$ otherwise. The matrix of node embeddings learned by a model is denoted by $\mathbf{Z} \in \mathbb{R}^{n \times d}$, where $d$ is the embedding dimension. We denote the sets of training and testing edges as $\mathcal{D}_{\text{train}}$ and $\mathcal{D}_{\text{test}}$, respectively. The model parameters are denoted by $\boldsymbol{\theta}$, and its prediction for the sign of an edge $e$ is $f_{\boldsymbol{\theta}^*}(\mathcal{G})_e$. Finally, $\mathcal{L}_{\text{train}}$ is the training loss and $\mathcal{L}_{\text{atk}}$ is the attacker's objective function.

\subsection{Link Sign Prediction}

\begin{figure}[!t]
\centerline{\includegraphics[width=0.45\textwidth]{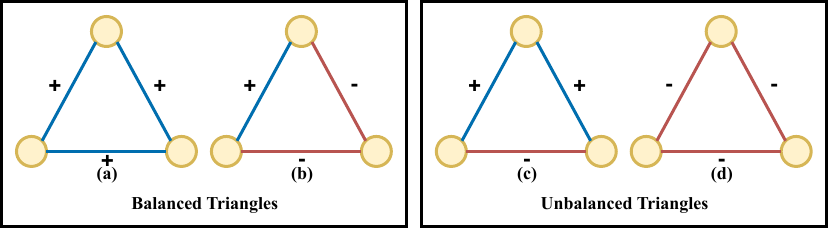}}
\caption{Balanced and unbalanced triangles. Positive and negative edges are represented by blue and red lines, respectively.}
\label{triangles}
\end{figure}

In this paper, we investigate the adversarial robustness of link sign prediction, a fundamental task in signed graph analysis. This task focuses on signed networks---where relationships are either positive or negative---and aims to infer the unknown sign of a given edge (e.g., $e_{uv}$) by leveraging the rest of the network's structure and existing edge signs. The practical significance of this problem is underscored by the prevalence of real-world networks with numerous edges whose relationship types are unobserved.

\subsection{Balance Theory}

Social balance theory \cite{zheng2015social} posits that individuals tend to form balanced social relationships, particularly in triadic formations. This theory is captured in common wisdom such as ``the friend of my friend is my friend" and ``the enemy of my enemy is my friend." In signed networks, triangular relationships are categorized as either balanced or unbalanced based on the parity of negative links they contain \cite{leskovec2010predicting,leskovec2010signed}. As illustrated in Fig.~\ref{triangles}, balanced triads are those with an even number of negative edges (zero or two). To quantify this property across an entire network, the balance degree $\mathcal{D}_3(\mathcal{G})$ \cite{cao2015grarep} is defined as the fraction of balanced triads:

\begin{equation}
\label{eqn:BalanceDegree}
    \mathcal{D}_3(\mathcal{G}) = \frac{\mathsf{Tr}(\mathbf{A}^3)+\mathsf{Tr}(|\mathbf{A}|^3)}{2\mathsf{Tr}(|\mathbf{A}|^3)},
\end{equation}
where $\mathsf{Tr}(\cdot)$ is the matrix trace operator and $|\mathbf{A}|$ denotes the element-wise absolute value of $\mathbf{A}$. The balance degree of real-world signed graphs typically ranges from $0.85$ to $0.95$. Balance theory is a cornerstone of SGNNs; for example, SGCN \cite{derr2018signed} incorporates it into its aggregation scheme, while models like SGCL \cite{shu2021sgcl} and SDGNN \cite{huang2021sdgnn} leverage it for data augmentation or loss function design.

\section{Problem Statements}
This research addresses link sign prediction under adversarial conditions. We first formalize the threat model before defining the attack and defense problems.

\subsection{Threat Model}
\subsubsection{Attacker's goal}
The attacker's goal is to degrade the overall predictive performance of a target SGNN model through a global poisoning attack \cite{zugner_adversarial_2019}. To this end, the attacker manipulates the edge signs within the training graph. The resulting perturbed graph is then used to train the target SGNN, with the aim of compromising its generalization ability on unseen data.

\subsubsection{Attacker's knowledge}
We assume a black-box setting where the adversary has access to the training graph's topology and edge signs but cannot access the target model's architecture, parameters, or gradients. This scenario reflects realistic conditions where training data may be public, but model internals are proprietary.

\subsubsection{Attacker's capability}
To ensure the attack remains imperceptible, the adversary is constrained by a perturbation budget $\Delta$, limiting the total number of edge sign flips. This is formalized as $\lVert \mathbf{A} - \hat{\mathbf{A}} \rVert_{0} \leq \Delta$, where $\hat{\mathbf{A}}$ is the perturbed adjacency matrix and $\lVert \cdot \rVert_0$ is the entry-wise $\mathcal{L}_0$ norm counting the number of non-zero elements. Since our attack only flips the signs of existing edges (e.g., $1 \leftrightarrow -1$), it inherently preserves the underlying graph topology, including node degrees and connectivity. These constraints define the set of admissible perturbations $\Phi(\mathcal{G};\Delta)$.

\subsection{Problem of Attack}
The global, untargeted poisoning attack can be formalized as a bi-level optimization problem \cite{han2024cost,zhao2024untargeted}:
\begin{align}
\label{eqn:BilevelOptimization}
    & \mathop{\min}_{\hat{\mathcal{G}} \in \Phi(\mathcal{G};\Delta)} \mathcal{L}_{\text{atk}} = \sum_{(i,j) \in \mathcal{D}_{\text{test}}}\mathbb{I}\{f_{\boldsymbol{\theta} ^ *}(\hat{\mathcal{G}})_{ij} = s_{ij}\}, \\ 
    &s.t. \  \boldsymbol{\theta}^* = \mathop{\arg \min}_{\boldsymbol{\theta}} \mathcal{L}_{\text{train}}(f_{\boldsymbol{\theta}} (\hat{\mathcal{G}})),\nonumber
\end{align}
where the adversary modifies the graph to minimize the model's accuracy on the test set, and the model is subsequently trained on this poisoned graph.

\subsection{Problem of Defense}
Given a poisoned graph $\hat{\mathcal{G}}$, the defender's objective is to train a robust SGNN that achieves a prediction accuracy comparable to what would be achieved on the clean graph $\mathcal{G}$. The defender faces two key constraints: they must work exclusively with the poisoned graph $\hat{\mathcal{G}}$ without access to its clean version $\mathcal{G}$, and they have no information about the specific attack method used or the perturbation budget $\Delta$. This creates a challenging requirement to develop a model that maintains resilience against various attack strategies.

\section{Proposed Black-Box Attack}
\subsection{Formulation of the Black-Box Attack}

As the model architecture and test data are inaccessible in a black-box setting, directly optimizing Eq.~\eqref{eqn:BilevelOptimization} is infeasible. We therefore propose targeting the graph's balance degree as a proxy objective. Prior research \cite{zhang2023rsgnn} has shown that SGNNs struggle to learn effective node representations from unbalanced triangles. This suggests that disrupting graph balance is a potent strategy for compromising SGNN performance. Specifically, training the target model on a graph with a reduced balance degree should degrade its performance on the test set. Following this intuition, we reformulate the optimization problem from Eq.~\eqref{eqn:BilevelOptimization} to directly minimize the balance degree:

\begin{equation}
\label{eqn:BalanceOptimization}
    \mathop{\min}_{\hat{\mathbf{A}} \in \Phi(\mathbf{A};\Delta)} \mathcal{D}_3(\hat{\mathcal{G}}),
\end{equation}
where $\Phi(\mathbf{A};\Delta)$ is the set of admissible perturbed matrices.

\subsection{Greedy Attack Method}

Solving the NP-hard optimization problem in Eq.~\eqref{eqn:BalanceOptimization} is computationally intractable due to the discrete nature of edge signs \cite{diao2020approximation}. Therefore, to find an effective approximate solution, we propose an efficient greedy algorithm that uses gradients as a heuristic guide. The core idea is to iteratively compute the gradient of the objective function $\mathcal{D}_3(\hat{\mathcal{G}})$ with respect to the current adjacency matrix $\hat{\mathbf{A}}$, and then flip the sign of the edge that provides the largest estimated decrease in the balance degree.

Standard gradient descent is inapplicable to the discrete adjacency matrix. For instance, updating a positive edge ($A_{ij}=1$) with a negative gradient is not a valid sign flip. Our method therefore identifies candidate edges for flipping by aligning the sign of the edge with the sign of its corresponding gradient. In each iteration, we select the candidate edge with the largest gradient magnitude to flip, repeating this process until the budget $\Delta$ is exhausted. The edge selection rule is formalized as:
% \begin{equation}
% \label{eqn:GreedyEdgeSignUpdate}
% \begin{aligned}
%     & (i^*,j^*) = \mathop{\arg \max}_{\{ (i,j) | \hat{a}_{ij}^{(k-1)} \neq 0 \wedge \text{sign}(\hat{a}_{ij}^{(k-1)}) = \atop \text{sign}(\nabla_{ij}\mathcal{D}_3(\hat{\mathcal{G}}^{(k-1)})) \}} |\nabla_{ij} \mathcal{D}_3(\hat{\mathcal{G}}^{(k-1)})|, \\
%     & \hat{a}_{i^*j^*}^{(k)} = -\hat{a}_{i^*j^*}^{(k-1)},
% \end{aligned}
% \end{equation}

\begin{equation}
\label{eqn:GreedyEdgeSignUpdate}
\begin{aligned}
    & (i^*,j^*) = \mathop{\arg \max}_{\{ (i,j) | \hat{a}_{ij}^{(k-1)} \cdot \nabla_{ij}\mathcal{D}_3(\hat{\mathcal{G}}^{(k-1)}) > 0 \}} |\nabla_{ij} \mathcal{D}_3(\hat{\mathcal{G}}^{(k-1)})|, \\
    & \hat{a}_{i^*j^*}^{(k)} = -\hat{a}_{i^*j^*}^{(k-1)},
\end{aligned}
\end{equation}
where $\hat{\mathbf{A}}^{(k)}$ is the adjacency matrix after $k$ flips, and $\nabla_{ij}$ is the gradient with respect to edge $(i,j)$. The detailed algorithm is outlined in Alg.~\ref{alg:1}.

\begin{algorithm}[!ht]
    \caption{Balance-Attack via Greedy Flips}
    \label{alg:1}
    \renewcommand{\algorithmicrequire}{\textbf{Input:}}
    \renewcommand{\algorithmicensure}{\textbf{Output:}}
    \begin{algorithmic}[1]
        \REQUIRE Original adjacency matrix $\mathbf{A}$, perturbation budget $\Delta$.
        \ENSURE Attacked adjacency matrix $\hat{\mathbf{A}}$.
        \STATE Initialize $\hat{\mathbf{A}} \gets \mathbf{A}$.
        \FOR{$k=1$ to $\Delta$}
            \STATE Calculate gradient matrix $\mathbf{G} \gets \nabla_{\hat{\mathbf{A}}} \mathcal{D}_3(\hat{\mathcal{G}})$.
            \STATE Find candidate edges $C_e = \{ (i,j) | \hat{a}_{ij} \neq 0 \wedge \text{sign}(\hat{a}_{ij}) = \text{sign}(G_{ij}) \}$.
            \IF{$C_e$ is empty}
                \STATE \textbf{break}
            \ENDIF
            \STATE Select $(i^*, j^*) = \mathop{\arg \max}_{(i,j) \in C_e} |G_{ij}|$.
            \STATE Flip edge sign: $\hat{a}_{i^*j^*} \gets -\hat{a}_{i^*j^*}$.
        \ENDFOR
        \STATE \textbf{return} $\hat{\mathbf{A}}$.
    \end{algorithmic}
\end{algorithm}

\section{Proposed Robust Model}
This section details our defense against the previously described threat model. We first analyze a naive defense to reveal a fundamental challenge, then present our robust model, BA-SGCL, which is designed to overcome this challenge.

\subsection{Our Preliminary Analysis}
As established in our analysis of balance-attack (Section V) and demonstrated by similar methods like FlipAttack \cite{zhu2024towards}, a common trait of these attacks is the significant reduction of the graph's balance degree. An intuitive first line of defense—analogous to structural learning for unsigned graphs—is an approach we term \textit{balance learning}. While structural learning refines the graph's topology, balance learning focuses exclusively on refining edge signs. Specifically, this approach treats the signs as learnable variables and iteratively updates them with the objective of maximizing the overall graph balance degree.

\begin{table}[!ht]
\centering
\caption{Comparison of SGCN without/with \textit{balance learning} under balance-attack (ratio: overlapping ratio of graphs; $\mathcal{D}_3$: balance degree).}
\resizebox{\linewidth}{!}{
\begin{tabular}{c c ccc ccc}
\toprule
\multirow{2}{*}{\textbf{Dataset}} & \multirow{2}{*}{\textbf{Ptb(\%)}} & \multicolumn{3}{c}{\textbf{SGCN}} & \multicolumn{3}{c}{\textbf{SGCN + \textit{balance learning}}} \\
\cmidrule(lr){3-5} \cmidrule(lr){6-8}
& & \multicolumn{1}{c}{AUC} & \multicolumn{1}{c}{ratio(\%)} & \multicolumn{1}{c}{$\mathcal{D}_3$} & \multicolumn{1}{c}{AUC} & \multicolumn{1}{c}{ratio(\%)} & \multicolumn{1}{c}{$\mathcal{D}_3$} \\
\midrule
\multirow{2}{*}{BitcoinAlpha} & 10 & 0.6917          & \textbf{89.98} & 0.2006 & \textbf{0.6962} & 84.77          & \textbf{0.9856} \\
                              & 20 & \textbf{0.6532} & \textbf{79.94} & 0.1054 & 0.6153          & 63.82          & \textbf{0.9616} \\
\cmidrule{1-8}
\multirow{2}{*}{BitcoinOTC}   & 10 & \textbf{0.7508} & \textbf{89.99} & 0.2072 & 0.7324          & 79.32          & \textbf{0.9598} \\
                              & 20 & \textbf{0.6982} & \textbf{79.98} & 0.0881 & 0.6687          & 64.92          & \textbf{0.9335} \\
\cmidrule{1-8}
\multirow{2}{*}{Slashdot}     & 10 & \textbf{0.6897} & \textbf{89.78} & 0.2345 & 0.6668          & 84.13          & \textbf{0.9436} \\
                              & 20 & \textbf{0.6344} & \textbf{79.96} & 0.1472 & 0.6092          & 68.23          & \textbf{0.9031} \\
\cmidrule{1-8}
\multirow{2}{*}{Epinions}     & 10 & \textbf{0.7387} & \textbf{89.97} & 0.3889 & 0.7253          & 88.58          & \textbf{0.9384} \\
                              & 20 & \textbf{0.6885} & \textbf{79.83} & 0.2197 & 0.6824          & 75.92          & \textbf{0.9081} \\
\bottomrule
\end{tabular}
}
\label{balance_learning_balance-attack_short}
\end{table}

We evaluated the performance of balance learning under balance-attack, with results shown in Table \ref{balance_learning_balance-attack_short}. The findings are striking: while balance learning effectively restores a high balance degree ($\mathcal{D}_3$), it consistently fails to improve, and often degrades, the model's predictive performance (AUC). Moreover, balance learning reduces the overlapping ratio between poisoned and clean graphs, indicating its inability to accurately recover original edge signs.

The failure of balance learning primarily stems from the fact that disparate sign distributions can yield identical balance degrees (illustrated by a toy example in Fig. \ref{Irreversibility}). We term this phenomenon the ``Irreversibility of Balance-related Information," which renders it challenging to accurately reconstruct the original sign distribution of the clean graph using balance degree as the sole guiding metric. As evidenced in Table \ref{balance_learning_balance-attack_short}, despite the restoration of a high balance degree, the resultant sign distribution remains significantly divergent from that of the clean graph. This insight motivates the need for a more sophisticated defense mechanism.

\begin{figure}[!t] \centering
\centerline{\includegraphics[width=0.4\textwidth]{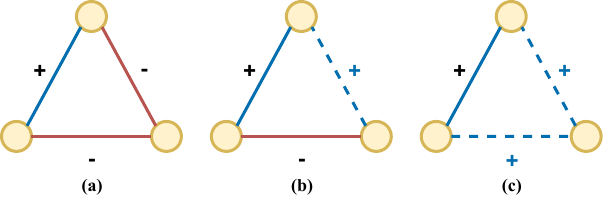}}
\caption{An example of the Irreversibility of Balance-related Information challenge. (a) The initial balanced graph; (b) The unbalanced graph after attack; (c) A recovered graph, which is balanced but has a different sign distribution from the graph in (a).}
\label{Irreversibility}
\end{figure}

\begin{figure*}[!t]
    \centering
    \includegraphics[width=0.8\textwidth]{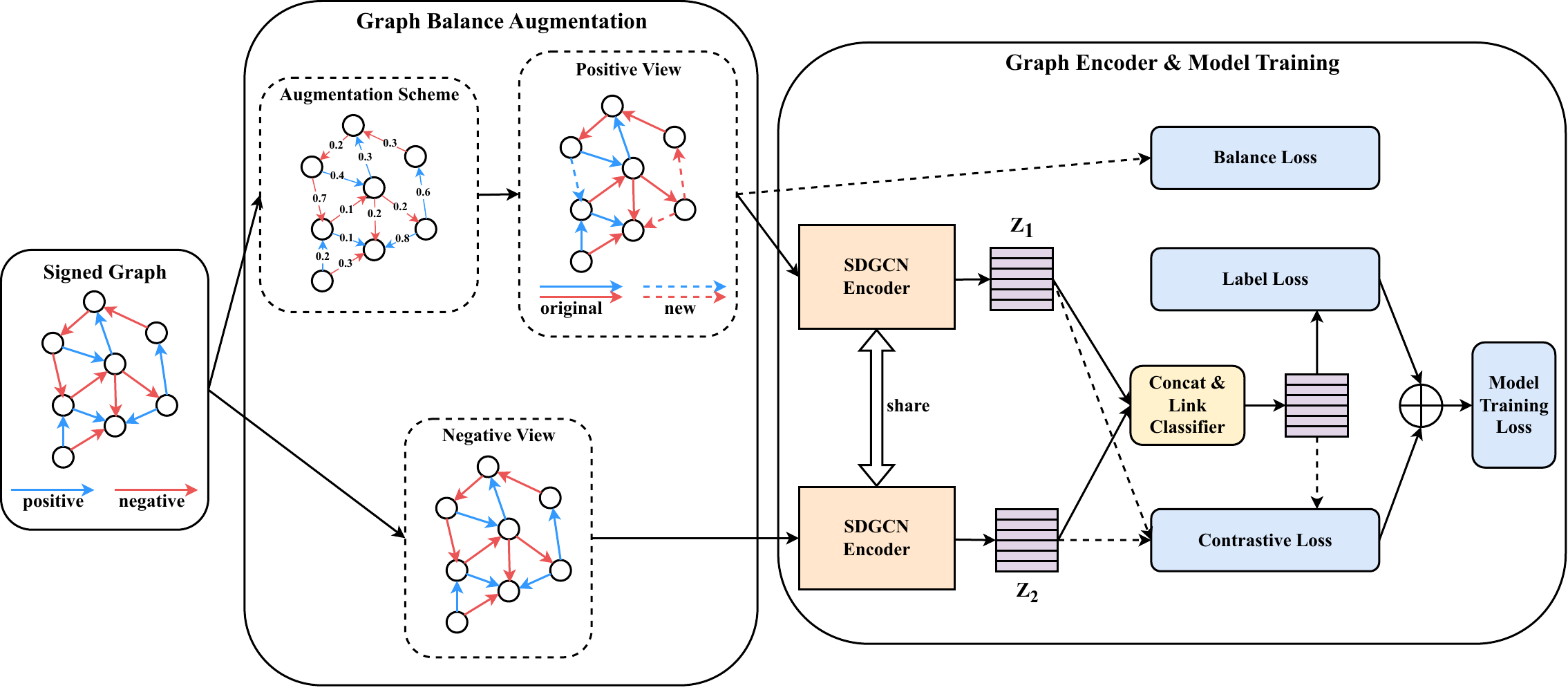}
    \captionsetup{justification=centering}
    \caption{The Overview of BA-SGCL.}
    \label{basgcl-overview}
\end{figure*}

\subsection{Method Overview}
The primary challenge in defending SGNNs is the ``Irreversibility of Balance-related Information,'' which makes directly recovering the original clean graph an ineffective strategy. To overcome this, we propose \textbf{Balance Augmented-Signed Graph Contrastive Learning (BA-SGCL)}, a model that bypasses direct graph recovery. Instead, its core idea is to learn robust node embeddings that are implicitly characterized by a high degree of balance.  BA-SGCL achieves this through a novel contrastive framework: it treats the input poisoned graph (low balance) as a negative view and uses a learnable augmenter to generate a corresponding positive view with an enhanced balance degree. A shared encoder is then trained to maximize the agreement between these two views. This process, jointly optimized with a supervised loss for the primary prediction task, forces the model to learn representations that are resilient to adversarial perturbations, as illustrated in Fig.~\ref{basgcl-overview}. In the following sections, we detail the model, explore the theoretical foundations of the attacks, and discuss the theoretical intuitions, framed from an information-theoretic perspective, that underpin our model's design for robust representation learning.

\subsection{Learnable Balance Augmentation}
Despite the difficulty in accurately restoring the clean graph to achieve a highly balanced state, we mitigate this issue via balance augmentation leveraging a GCL framework \cite{zhu2021graph,suresh2021adversarial}. GCL primarily relies on generating pairs of positive and negative views for self-supervised learning. In the context of defending against poisoning attacks, we only have access to the poisoned graph $\hat{\mathbf{A}}$, where the attacker has already reduced the balance degree. This poisoned graph can serve as the negative view. To generate a positive view with an enhanced balance degree, we introduce a novel balance augmentation technique that involves flipping edge signs on the poisoned graph.

Specifically, we learn a Bernoulli distribution to determine which edge sign flips are likely to increase the balance degree. Let $\mathbf{\Delta}=[\Delta_{ij}]_{n \times n}\in [0, 1]^{n \times n}$ denote the probability matrix for flipping signs. The key to our balance augmentation is learning the optimal $\mathbf{\Delta}$. We represent the Bernoulli distribution for flipping edge $(i,j)$ as $\mathcal{B}(\Delta_{ij})$. Then, we can sample a sign perturbation mask $\mathbf{E} \in \{0, 1\}^{n \times n}$, where $E_{ij} \sim \mathcal{B}(\Delta_{ij})$ indicates whether to flip the sign of edge $(i,j)$. The adjacency matrix of the sampled augmented positive view, $\mathbf{A}_p$, can be represented as follows:
\begin{equation}
\label{eqn:SignedAugmentation}
    \mathbf{A}_p = \hat{\mathbf{A}} + \mathbf{C} \circ \mathbf{E}, \quad \text{where} \quad \mathbf{C} = -2 \hat{\mathbf{A}}.
\end{equation}
The probability matrix $\mathbf{\Delta}$ is learned by minimizing the negative balance degree of the expected augmented graph. This objective, which we will later refer to as the balance loss, is defined as:
\begin{equation}
\label{eqn:balance_objective_function}
    \mathcal{L}_{\text{balance}} = -\mathcal{D}_3(\mathbb{E}[\mathbf{A}_p]) = -\mathcal{D}_3(\hat{\mathbf{A}} - 2\hat{\mathbf{A}} \circ \mathbf{\Delta}).
\end{equation}
Minimizing this term aims to find a probability distribution $\mathbf{\Delta}$ that, on average, generates a positive view with the highest possible balance degree. To preserve the integrity of the graph structure during this defense-side augmentation, we introduce a constraint set $\Phi_D(\hat{\mathbf{A}})$ that limits the maximum number of edge sign flips. In practice, we implement this budget by selecting the top $n_D\%$ of $\Delta_{ij}$ values to sample $E_{ij}$, where $n_D\%$ is a hyperparameter.

\subsection{Design of Loss Function}
The overall objective function for training BA-SGCL comprises three key components: the contrastive loss $\mathcal{L}_{\text{con}}$ (for learning robust representations by comparing different graph views), the label loss $\mathcal{L}_{\text{label}}$ (for the primary task of link sign prediction), and the balance loss $\mathcal{L}_{\text{balance}}$ (for optimizing the learnable augmentation scheme). The encoder's parameters $\boldsymbol{\theta}$ are optimized using a combination of $\mathcal{L}_{\text{con}}$ and $\mathcal{L}_{\text{label}}$, while $\mathcal{L}_{\text{balance}}$ is specifically used to train the probability matrix $\mathbf{\Delta}$ that governs the balance augmentation.

\subsubsection{Contrastive loss}
The contrastive objective aims to align latent representations of identical nodes while differentiating them from other nodes. Nodes that are identical across different graph views constitute inter-positive pairs, while all other node combinations form inter-negative pairs. For instance, node $u$ in $\mathcal{G}_1$ and its counterpart in $\mathcal{G}_2$ create an inter-positive pair. Conversely, node $u$ in $\mathcal{G}_1$ and any node ${v \in \mathcal{V}; v \neq u}$ in $\mathcal{G}_2$ form an inter-negative pair. The inter-view objective seeks to maximize similarity within positive pairs while minimizing it among negative pairs. The inter-view loss function is defined as:
\begin{equation}
    \mathcal{L}_{\text{inter}} = \frac{1}{|\mathcal{V}|} \sum_{u \in \mathcal{V}} \log \frac{\exp((\mathbf{z}_1^u \cdot \mathbf{z}_2^u)/\tau)}{\sum_{v \in \mathcal{V}} \exp((\mathbf{z}_1^u \cdot \mathbf{z}_2^v)/\tau)}.
    \label{L_inter}
\end{equation}
where $\mathbf{z}_1^u$ denotes the low-dimensional embedding vector of node $u$ in view $1$, while $\mathbf{z}_2^u$ represents the corresponding embedding in view $2$.

Unlike the inter-view loss which compares node representations across different graph views, the intra-view loss calculates the discriminative loss within a single graph view. It is crucial for ensuring that each node's latent representation is distinct, reflecting its unique characteristics. The primary goal is to enhance the differentiation among all node representations. The intra-view loss is mathematically defined as:
\begin{equation}
    \mathcal{L}_{\text{intra}} = \frac{1}{K} \sum_{k=1}^{K} \frac{1}{|\mathcal{V}|} \sum_{u \in \mathcal{V}} \log \frac{1}{\sum_{v \in \mathcal{V},u \neq v} \exp((\mathbf{z}_k^u \cdot \mathbf{z}_k^v)/\tau)},
    \label{L_intra}
\end{equation}
where $k$ indicates the graph view index.

The contrastive loss is the sum of the inter-view and intra-view loss functions:
\begin{equation}
    \mathcal{L}_{\text{con}} = \mathcal{L}_{\text{inter}} + \mathcal{L}_{\text{intra}}.
    \label{L_con}
\end{equation}

\subsubsection{Label Loss}
The graph encoder processes the two views (the poisoned graph and the augmented positive view), producing node representations $\mathbf{Z}_1$ and $\mathbf{Z}_2$. These representations are then concatenated and passed through an output layer to yield the final node embeddings $\mathbf{R}$ for prediction tasks:
\begin{equation}
    \mathbf{R} = \sigma([\mathbf{Z}_1 || \mathbf{Z}_2]\mathbf{W}^{\text{out}}+\mathbf{B}^{\text{out}}).
    \label{embedding}
\end{equation}
Specifically, after generating the final representations $\mathbf{r}_u \in \mathbf{R}$ for all nodes, we utilize a 2-layer MLP to predict the raw sign score $\hat{s}_{ij}$ for an edge $(i,j)$:
\begin{equation}
    \hat{s}_{ij} = \text{MLP}([\mathbf{r}_i || \mathbf{r}_j]).
\end{equation}
The loss function for link sign prediction is the binary cross-entropy. For this purpose, the ground truth signs $sign(e_{ij}) \in \{+, -\}$ are mapped to binary labels $y_{ij} \in \{1, 0\}$ (e.g., $1$ for a positive sign and $0$ for a negative sign). The loss is then defined as:
\begin{equation}
\begin{aligned}
    \mathcal{L}_{\text{label}} = & -\frac{1}{|\Omega_{\text{train}}|} \sum_{(i,j) \in \Omega_{\text{train}}} [y_{ij}\log\sigma(\hat{s}_{ij}) \\
    & + (1-y_{ij})\log(1-\sigma(\hat{s}_{ij}))],
    \label{L_label}
\end{aligned}
\end{equation}
where $\Omega_{\text{train}}$ is the set of labeled edges in the training data, and $\sigma(\cdot)$ is the sigmoid function.

\subsubsection{Balance Loss}
The balance loss $\mathcal{L}_{\text{balance}}$ is identical to $\mathcal{L}_{\text{ptb}}$ defined in Eq.~\eqref{eqn:balance_objective_function} (i.e., the negative balance degree). This loss is specifically used to optimize the parameters of the learnable augmenter, namely the probability matrix $\mathbf{\Delta}$, to encourage the generation of positive views with an enhanced balance degree.

\begin{algorithm}[!t]
\caption{BA-SGCL Training Algorithm}
\label{alg:basgcl}
    \renewcommand{\algorithmicrequire}{\textbf{Input:}}
    \renewcommand{\algorithmicensure}{\textbf{Output:}}
    \begin{algorithmic}[1]
        \REQUIRE Poisoned adjacency matrix $\hat{\mathbf{A}}$, training labels $\Omega_{\text{train}}$, hyperparameters $\alpha, \lambda_{\text{intra}}, n_D\%$.
        \ENSURE Trained encoder parameters $\boldsymbol{\theta}$.
        \STATE Initialize encoder parameters $\boldsymbol{\theta}$ and probability matrix $\mathbf{\Delta}$.
        \FOR{each training epoch}
            \STATE // Generate positive and negative views
            \STATE Sample perturbation mask $\mathbf{E}$ from $\mathcal{B}(\mathbf{\Delta})$ with budget $n_D\%$.
            \STATE Construct positive view $\mathbf{A}_p \leftarrow \hat{\mathbf{A}} - 2\hat{\mathbf{A}} \circ \mathbf{E}$.
            \STATE Let negative view $\mathbf{A}_n \leftarrow \hat{\mathbf{A}}$.
            \STATE 
            \STATE // Pass views through the shared encoder
            \STATE $\mathbf{Z}_n \leftarrow f_{\boldsymbol{\theta}}(\mathbf{A}_n)$, $\mathbf{Z}_p \leftarrow f_{\boldsymbol{\theta}}(\mathbf{A}_p)$.
            \STATE
            \STATE // Compute losses by referencing their definitions
            \STATE Compute contrastive loss $\mathcal{L}_{\text{con}}$ via Eqs.~\eqref{L_inter}, \eqref{L_intra}, and \eqref{L_con}.
            \STATE Compute label loss $\mathcal{L}_{\text{label}}$ via Eq.~\eqref{L_label}.
            \STATE Compute overall model loss $\mathcal{L}$ via Eq.~\eqref{L_model_training}.
            \STATE Compute balance loss $\mathcal{L}_{\text{balance}}$ via Eq.~\eqref{eqn:balance_objective_function}.
            \STATE
            \STATE // Update parameters concurrently
            \STATE Update encoder $\boldsymbol{\theta}$ using $\nabla_{\boldsymbol{\theta}}\mathcal{L}$.
            \STATE Update augmenter $\mathbf{\Delta}$ using $\nabla_{\mathbf{\Delta}}\mathcal{L}_{\text{balance}}$.
        \ENDFOR
        \STATE \textbf{return} $\boldsymbol{\theta}$.
    \end{algorithmic}
\end{algorithm}

\subsection{Model Training}
For our encoder, we employ SDGCN \cite{ko2023spectral}, the current state-of-the-art SGNN encoder. SDGCN distinguishes itself by overcoming the constraints of traditional graph Laplacians and leveraging complex number representations to capture both sign and directional information of edges in signed graphs.

Contrastive learning can be conceptualized as a regularization mechanism for the target task. Consequently, we update the model encoder's parameters using a combined objective, defined as:
\begin{equation}
    \mathcal{L} = \alpha \times \mathcal{L}_{\text{con}} + \mathcal{L}_{\text{label}}.
    \label{L_model_training}
\end{equation}
The training process involves jointly optimizing the encoder and the learnable augmenter. In each training iteration, two update steps are performed concurrently: 1) The encoder's parameters $\boldsymbol{\theta}$ are updated using the combined loss $\mathcal{L}$ from Eq.~\eqref{L_model_training}. 2) The augmenter's parameters, the probability matrix $\mathbf{\Delta}$, are updated using the balance loss $\mathcal{L}_{\text{balance}}$. This unified approach avoids the significant computational overhead of a sequential, multi-stage training process. The complete algorithm is detailed in Alg.~\ref{alg:basgcl}.

\section{Theoretical Analysis}
This section examines adversarial attacks from a mutual information (MI) perspective \cite{kraskov2004estimating,xu2021infogcl} and establishes the theoretical underpinnings of our proposed defense framework.

\begin{proposition}[Attack Principle from an MI Perspective]
\label{prop:1}
Balance-related adversarial attacks aim to degrade model performance by corrupting the graph's balance information, thereby minimizing the mutual information between this information and the ground-truth labels.
\end{proposition}

\begin{proof}[Proof Sketch]
Let us conceptualize the information in a signed graph $\mathcal{G}$ as two channels: the topological structure (which edges exist), represented by an unweighted adjacency matrix $\mathbf{A}_{\text{abs}}$, and the sign information, $\mathbf{S}$. The balance-related information, $\mathcal{B}$, is a function of both. An SGNN model $f_{\boldsymbol{\theta}}$ aims to maximize the MI between its learned embeddings and the true labels $Y$, i.e., $\max_{\boldsymbol{\theta}} I(f_{\boldsymbol{\theta}}(\mathbf{A}_{\text{abs}}, \mathbf{S}); Y)$.

A balance-attack perturbs the signs $\mathbf{S}$ to $\hat{\mathbf{S}}$ while keeping the topology $\mathbf{A}_{\text{abs}}$ fixed. This changes the balance information from $\mathcal{B}$ to $\hat{\mathcal{B}}$. The attacker's objective is to minimize the model's performance:
\begin{equation}
    \mathop{\arg\min}_{\hat{\mathbf{S}}} I(f_{\boldsymbol{\theta}}(\mathbf{A}_{\text{abs}}, \hat{\mathbf{S}}); Y).
\end{equation}
Let $h_{\mathcal{A}} = g_1(\mathbf{A}_{\text{abs}})$ and $h_{\hat{\mathcal{B}}} = g_2(\hat{\mathbf{S}})$ be representations of the structure and the perturbed balance information, respectively. Using the chain rule for mutual information, we can decompose the objective:
\begin{equation}
\label{eq:mi_chain_rule_attack}
    I((h_{\mathcal{A}}, h_{\hat{\mathcal{B}}});Y) = I(h_{\mathcal{A}};Y) + I(h_{\hat{\mathcal{B}}};Y|h_{\mathcal{A}}).
\end{equation}
During the attack, the topology $\mathbf{A}_{\text{abs}}$ and the labels $Y$ are fixed. Therefore, the term $I(h_{\mathcal{A}};Y)$ is a constant. To minimize the entire expression, the attacker must focus on minimizing the second term. Thus, the attack objective simplifies to:
\begin{equation}
    \mathop{\arg\min}_{\hat{\mathbf{S}}} I(h_{\hat{\mathcal{B}}};Y|h_{\mathcal{A}}).
\end{equation}
This confirms that the attack's essence is to minimize the conditional MI between the perturbed balance information and the labels, given the fixed graph structure.
\end{proof}

\begin{proposition}[BA-SGCL Defense Principle]
\label{prop:2}
BA-SGCL learns robust embeddings by jointly optimizing lower bounds on two mutual information objectives: (1) the MI between the poisoned graph view and the high-balance augmented view, and (2) the MI between the final embeddings and the ground-truth labels.
\end{proposition}

\begin{proof}[Proof Sketch]
Our defense framework is trained by minimizing two main loss components, which can be theoretically linked to maximizing two MI objectives. Let $\mathbf{Z}_1$ be the embeddings from the positive view (high-balance augmented graph) and $\mathbf{Z}_2$ be the embeddings from the negative view (poisoned graph).

First, the contrastive loss $\mathcal{L}_{\text{con}}$ aims to maximize the MI between the two views, $I(\mathbf{Z}_1; \mathbf{Z}_2)$. The InfoNCE estimator used in $\mathcal{L}_{\text{inter}}$ provides a variational lower bound on this mutual information \cite{oord2018representation}. Specifically, the relationship is given by:
\begin{equation}
    I(\mathbf{Z}_1; \mathbf{Z}_2) \geq \log(n) - \mathcal{L}_{\text{inter}}.
\end{equation}
Therefore, by minimizing the contrastive loss $\mathcal{L}_{\text{con}}$, we are effectively maximizing a lower bound on the mutual information between the two views. This forces the encoder to learn representations from the poisoned graph ($\mathbf{Z}_2$) that are invariant to the balance-restoring augmentations in the positive view ($\mathbf{Z}_1$), thereby capturing robust, high-balance characteristics.

Second, the label loss $\mathcal{L}_{\text{label}}$ aims to make the final representations predictive. Let $\mathbf{R}$ be the final embedding used for prediction (as defined in Eq.~\eqref{embedding}). Minimizing the binary cross-entropy loss is equivalent to maximizing the conditional log-likelihood of the labels given the representation, $\log P(Y|\mathbf{R})$. This objective is known to maximize a lower bound on the mutual information between the representation and the labels, $I(\mathbf{R}; Y)$ \cite{poole2019variational}.

By jointly minimizing $\mathcal{L} = \alpha \mathcal{L}_{\text{con}} + \mathcal{L}_{\text{label}}$, our BA-SGCL framework effectively maximizes the lower bounds of both MI terms. This ensures that the learned embeddings are simultaneously robust to adversarial perturbations (via contrastive learning) and predictive for the downstream task (via supervised learning).
\end{proof}

\section{Experiments}

We first evaluate the effectiveness of balance-attack against $9$ state-of-the-art signed graph representation methods. Our investigation addresses the following research questions:

\begin{itemize}
\item \textbf{Q1}: How effectively does balance-attack reduce the structural balance in signed graphs?
\item \textbf{Q2}: How does balance-attack's performance compare to random attack baselines when targeting state-of-the-art SGNN models?
\item \textbf{Q3}: To what extent is balance-attack generalizable across different SGNN architectures and frameworks?
\end{itemize}

The second set of experiments evaluates link sign prediction performance, comparing our proposed BA-SGCL against nine state-of-the-art SGNN baselines under various signed graph adversarial attacks, including balance-attack. For this analysis, we focus on:

\begin{itemize}
\item \textbf{Q4}: How does BA-SGCL's robustness compare to existing SGNN methods when subjected to different signed graph adversarial attacks?
\item \textbf{Q5}: What advantages does balance augmentation in BA-SGCL offer over the random augmentation strategy employed in standard SGCL?
\end{itemize}

\begin{table}[h] \centering
\caption{Dataset Statistics.}
\resizebox{\linewidth}{!}{
\begin{tabular}{cccccc}
\toprule
\textbf{Dataset} & \textbf{\#Nodes} & \textbf{\#Pos-Edges} & \textbf{\#Neg-Edges} & \textbf{\%Pos-Ratio} & \textbf{\%Density} \\
\midrule
Bitcoin-Alpha & 3,784  & 22,650  & 1,536   & 93.65 & 0.3379 \\
Bitcoin-OTC   & 5,901  & 32,029  & 3,563   & 89.99 & 0.2045 \\
Slashdot      & 33,586 & 295,201 & 100,802 & 74.55 & 0.0702 \\
Epinions      & 16,992 & 276,309 & 50,918  & 84.43 & 0.2266 \\
\bottomrule
\end{tabular}}
\label{dataset}
\end{table}

\subsection{Datasets}

Experiments are conducted on $4$ public real-world datasets: Bitcoin-Alpha, Bitcoin-OTC \cite{kumar2016edge}, Epinions \cite{guha2004propagation}, and Slashdot \cite{guo2017hermitian}. The Bitcoin-Alpha and Bitcoin-OTC datasets, collected from Bitcoin trading platforms, are publicly available. These platforms allow users to label others as either trustworthy (positive) or untrustworthy (negative), serving as a mechanism to prevent fraudulent transactions in these anonymous trading environments. Slashdot, a prominent technology-focused news website, features a unique user community where members can designate others as friends or foes based on their interactions. Epinions, an online social network centered around a consumer review site, allows users to establish trust relationships with other members.

For experiments, $80\%$ of the links are randomly selected as the training set, with the remaining $20\%$ serving as the test set. As the datasets lack attributes, each node is assigned a randomly generated 64-dimensional vector as its initial attribute. Detailed dataset statistics are presented in Table \ref{dataset}.

\subsection{Setup}
The experimental settings are divided into two subsections.

\subsubsection{Attack Setup}
Following established practices in the signed graph literature, we set the embedding dimension at $64$ for all SGNN models to ensure a fair comparison. We evaluate two attack strategies. Our proposed balance-attack is a gradient-guided greedy algorithm that iteratively flips the sign of the edge estimated to cause the largest decrease in balance degree. This is compared against a random attack baseline, which flips the signs of an equal number of edges selected uniformly at random. For both attacks, we examine perturbation rates ranging from $5\%$ to $20\%$ of the total edges. Model performance is evaluated using four standard metrics from the signed graph literature \cite{huang2021sdgnn, zhang2023rsgnn}: AUC, Micro-F1, Binary-F1, and Macro-F1. Lower values for these metrics indicate more severe performance degradation and thus stronger attack effectiveness.

\subsubsection{Defense Setup}
For our BA-SGCL implementation in PyTorch, we set the learning rate to $0.001$. For the learnable balance augmenter, we set the defense augmentation budget ($n_D\%$) to a value that typically results in the positive view's balance degree exceeding $0.9$. We preserve the original parameter settings for all baseline methods to ensure a fair comparison. We employ the same evaluation metrics as in the attack experiments; however, higher values now indicate superior defense performance. In all result tables, bold and underlined values represent the best and second-best performances, respectively.

\subsection{Baselines}
The baselines are divided into two parts, focused on attacks and defenses, respectively.
\subsubsection{Victim Models}
We evaluate the effectiveness of balance-attack on the following nine state-of-the-art SGNN models, which serve as the victims in this experiment:

\begin{itemize}
\item \textbf{SiNE} \cite{wang2017signed} is a signed graph embedding method that uses deep neural networks and an extended structural balance theory-based loss function.
\item \textbf{SGCN} \cite{derr2018signed} introduces a novel information aggregator based on balance theory, expanding the application of GCN to signed graphs.
\item \textbf{SNEA} \cite{li2020learning} generalizes the graph attention network (GAT) \cite{velivckovic2018graph} to signed graphs and is also based on balance theory.
\item \textbf{BESIDE} \cite{chen2018bridge} combines balance and status theory. It utilizes status theory to learn ``bridge" edge information and combines it with triangle information.
\item \textbf{SGCL} \cite{shu2021sgcl} is the first work to generalize GCL to signed graphs.
\item \textbf{SDGNN} \cite{huang2021sdgnn} combines balance theory and status theory, and introduces four weight matrices to aggregate neighbor features based on edge types.
\item \textbf{RSGNN} \cite{zhang2023rsgnn} improves SGNN performance by using structure-based regularizers to highlight the intrinsic properties of signed graphs and reduce vulnerability to input graph noise.
\item \textbf{SDGCN} \cite{ko2023spectral} defines a spectral graph convolution encoder with a magnetic Laplacian.
\item \textbf{UGCL} \cite{ko2023universal} presents a GCL framework that incorporates Laplacian perturbation.
\end{itemize}

\subsubsection{Attacks for Robustness Evaluation}
We evaluate the robustness of our proposed BA-SGCL and compare it against the performance of the nine baseline SGNNs under attack. The evaluation is conducted using perturbations from two attack methods: our balance-attack \cite{zhou2024black} and FlipAttack \cite{zhu2024towards}. FlipAttack employs bi-level optimization with conflicting metrics as penalties to generate stealthy perturbations while compromising model performance. Both attacks modify edge signs while maintaining graph structure. Due to the computational cost of these attack methods on large graphs, we evaluate on $2000$-node subgraphs sampled from the Slashdot and Epinions datasets, preserving the original train-test splits.

\subsection{Effectiveness in Reducing Graph Balance Degree (Q1)}
We first evaluate the effectiveness of our method in reducing graph balance degree compared to random attacks. Fig. \ref{balance_degree} presents the comparative results between balance-attack and random attacks. The initial balance degree across all datasets ranges from $0.85$ to $0.9$. Under random attacks with $20\%$ perturbation rate, the balance degree only decreases to around $0.65$. In contrast, our balance-attack achieves substantially lower balance degrees: with merely $5\%$ perturbation rate, it reduces the balance degree to $0.35$-$0.55$, and with $20\%$ perturbation rate, it further decreases the balance degree to approximately $0.1$. These results demonstrate that our proposed method significantly outperforms random attacks in reducing graph balance degree.

\begin{figure}
\centerline{\includegraphics[width=0.45\textwidth]{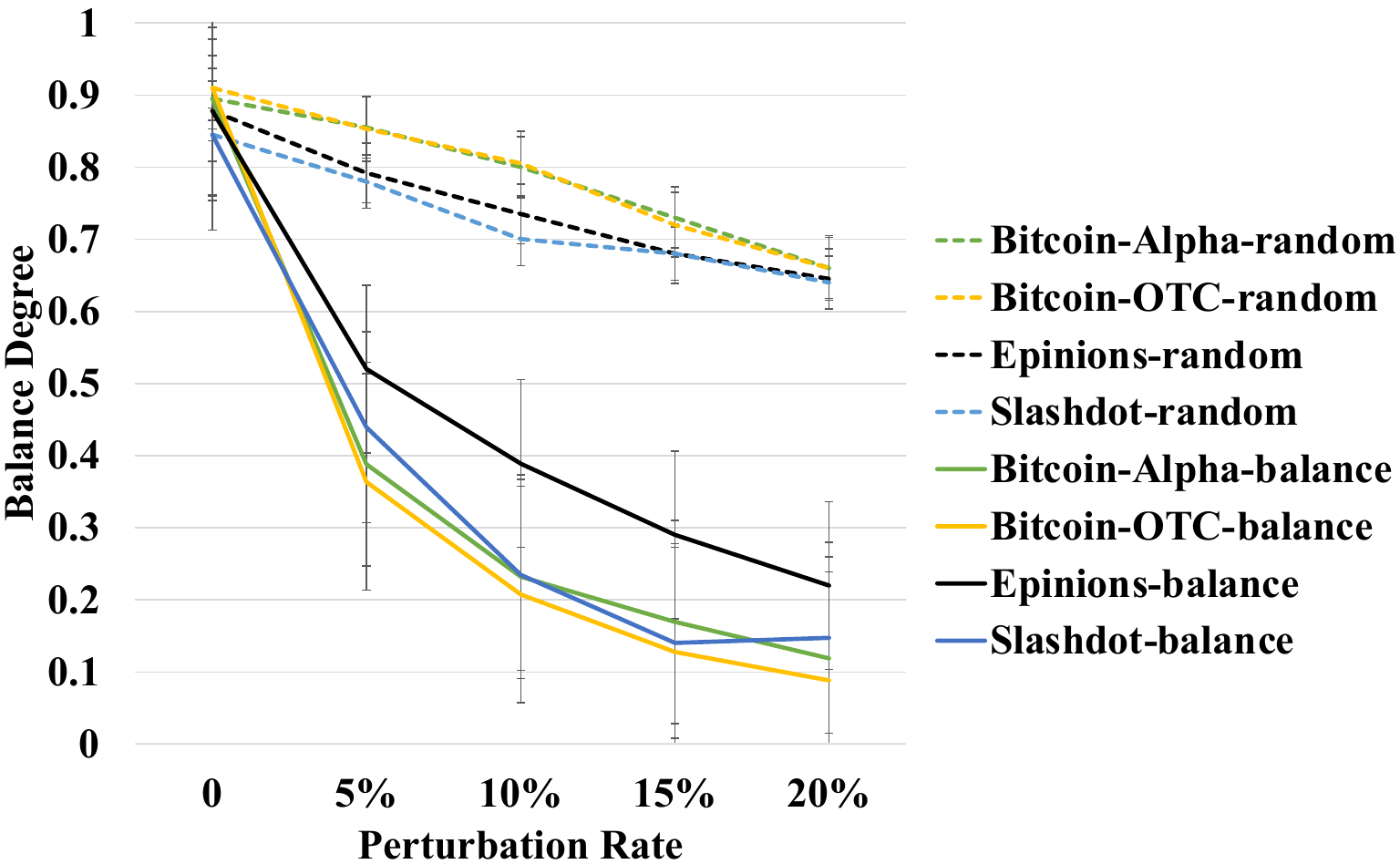}}
\caption{Balance degree of $4$ datasets under random attacks and balance-attack.}
\label{balance_degree}
\end{figure}

\subsection{Attack Performance Against State-of-the-Art SGNNs (Q2)}

To answer Q2, we conduct a comprehensive comparison between our proposed balance-attack and a standard random attack baseline. The evaluation is performed on a representative subset of five state-of-the-art SGNN models, with detailed results under a 20\% perturbation rate presented in Table~\ref{other_models}. 

The results clearly demonstrate the superior effectiveness of balance-attack. Across all five models, our method induces a significantly more substantial performance degradation than the random baseline. This is particularly evident in the case of RSGNN, a model specifically designed for robustness against random noise. While the random attack has a relatively modest impact on RSGNN's performance, its accuracy drops sharply when subjected to our balance-attack. This confirms that targeting the graph's balance degree is a critical vulnerability, capable of defeating even models that are robust to other forms of random perturbations. The broad effectiveness of balance-attack across all nine SGNN models is further detailed in the main result tables (Tables~\ref{balance-attack_performance_auc&macro} and \ref{balance-attack_performance_micro&binary}).

\begin{table}[ht] \centering
\caption{Link sign prediction performance of four representative SGNNs under random attack and Balance-attack with perturbation rate = $20\%$.}
\resizebox{\linewidth}{!}{
\begin{tabular}{cccccc}
\toprule
\textbf{Model} & \textbf{Dataset} & \textbf{Attack} & \textbf{Macro-F1} & \textbf{Micro-F1} & \textbf{Binary-F1} \\
\midrule
\multirow{8}{*}{SGCL} & \multirow{2}{*}{Bitcoin-Alpha} & random            & 0.6007          & 0.9305          & 0.9636          \\
                      &                   & \textbf{balance}  & \textbf{0.5317} & \textbf{0.8054} & \textbf{0.8734} \\
                      \cmidrule(lr){2-6}
                      & \multirow{2}{*}{Bitcoin-OTC}   & random            & 0.6131          & 0.9026          & 0.9480          \\
                      &                   & \textbf{balance}  & \textbf{0.5712} & \textbf{0.7734} & \textbf{0.8662} \\
                      \cmidrule(lr){2-6}
                      & \multirow{2}{*}{Slashdot}      & random            & 0.5578          & 0.8338          & 0.9072          \\
                      &                   & \textbf{balance}  & \textbf{0.5005} & \textbf{0.7005} & \textbf{0.8166} \\
                      \cmidrule(lr){2-6}
                      & \multirow{2}{*}{Epinions}      & random            & \textbf{0.5673} & 0.8482          & 0.9160          \\
                      &                   & \textbf{balance}  & 0.5877          & \textbf{0.7382} & \textbf{0.8374} \\
\cmidrule{1-6}
\multirow{8}{*}{SDGNN} & \multirow{2}{*}{Bitcoin-Alpha} & random            & 0.6062          & 0.8616          & 0.9234          \\
                       &                   & \textbf{balance}  & \textbf{0.5124} & \textbf{0.7372} & \textbf{0.8395} \\
                       \cmidrule(lr){2-6}
                       & \multirow{2}{*}{Bitcoin-OTC}   & random            & 0.6593          & 0.8333          & 0.9028          \\
                       &                   & \textbf{balance}  & \textbf{0.5397} & \textbf{0.7085} & \textbf{0.8174} \\
                       \cmidrule(lr){2-6}
                       & \multirow{2}{*}{Slashdot}      & random            & 0.6966          & 0.8405          & 0.8981          \\
                       &                   & \textbf{balance}  & \textbf{0.5702} & \textbf{0.7323} & \textbf{0.8283} \\
                       \cmidrule(lr){2-6}
                       & \multirow{2}{*}{Epinions}      & random            & 0.6714          & 0.8336          & 0.9023          \\
                       &                   & \textbf{balance}  & \textbf{0.6362} & \textbf{0.7693} & \textbf{0.8554} \\
\cmidrule{1-6}
\multirow{8}{*}{RSGNN} & \multirow{2}{*}{Bitcoin-Alpha} & random            & 0.5165          & 0.6839          & 0.8010          \\
                       &                   & \textbf{balance}  & \textbf{0.4832} & \textbf{0.6305} & \textbf{0.7634} \\
                       \cmidrule(lr){2-6}
                       & \multirow{2}{*}{Bitcoin-OTC}   & random            & 0.6341          & 0.7828          & 0.8673          \\
                       &                   & \textbf{balance}  & \textbf{0.5198} & \textbf{0.6427} & \textbf{0.7622} \\
                       \cmidrule(lr){2-6}
                       & \multirow{2}{*}{Slashdot}      & random            & 0.6044          & 0.6637          & 0.7576          \\
                       &                   & \textbf{balance}  & \textbf{0.5212} & \textbf{0.6005} & \textbf{0.7162} \\
                       \cmidrule(lr){2-6}
                       & \multirow{2}{*}{Epinions}      & random            & 0.6492          & 0.7409          & 0.8285          \\
                       &                   & \textbf{balance}  & \textbf{0.5967} & \textbf{0.6835} & \textbf{0.7833} \\
\cmidrule{1-6}
\multirow{8}{*}{UGCL}  & \multirow{2}{*}{Bitcoin-Alpha} & random            & 0.6192          & 0.9199          & 0.9576          \\
                       &                   & \textbf{balance}  & \textbf{0.5422} & \textbf{0.7995} & \textbf{0.8754} \\
                       \cmidrule(lr){2-6}
                       & \multirow{2}{*}{Bitcoin-OTC}   & random            & 0.6983          & 0.8988          & 0.9442          \\
                       &                   & \textbf{balance}  & \textbf{0.5948} & \textbf{0.7755} & \textbf{0.8645} \\
                       \cmidrule(lr){2-6}
                       & \multirow{2}{*}{Slashdot}      & random            & 0.6318          & 0.8538          & 0.9173          \\
                       &                   & \textbf{balance}  & \textbf{0.5875} & \textbf{0.7823} & \textbf{0.8707} \\
                       \cmidrule(lr){2-6}
                       & \multirow{2}{*}{Epinions}      & random            & \textbf{0.6390} & 0.8635          & 0.9237          \\
                       &                   & \textbf{balance}  & 0.6462          & \textbf{0.8325} & \textbf{0.9015} \\
\bottomrule
\end{tabular}}
\label{other_models}
\end{table}

\subsection{Generalizability Across Different SGNN Architectures (Q3)}
To evaluate the generalizability of our attack, we examine its effectiveness on models that do not explicitly rely on balance theory, such as UGCL. As shown in Table~\ref{other_models}, balance-attack proves remarkably effective even against these models. This demonstrates that disrupting a graph's balance is a more fundamental vulnerability than previously understood, affecting a wide range of SGNN architectures, not just those explicitly based on balance theory.

\subsection{Defense Performance against Attacks (Q4)}

\begin{table*}[!t] \centering
\caption{AUC and Macro-F1 of SGNNs on link sign prediction under balance-attack.}
\resizebox{\linewidth}{!}{
\begin{tabular}{c c cccccccccccccccccccc}
\toprule
\multirow{2}{*}{\textbf{Dataset}} & \multirow{2}{*}{\textbf{Ptb(\%)}} & \multicolumn{2}{c}{\textbf{SiNE}} & \multicolumn{2}{c}{\textbf{SGCN}} & \multicolumn{2}{c}{\textbf{SNEA}} & \multicolumn{2}{c}{\textbf{BESIDE}} & \multicolumn{2}{c}{\textbf{SDGNN}} & \multicolumn{2}{c}{\textbf{SDGCN}} & \multicolumn{2}{c}{\textbf{RSGNN}} & \multicolumn{2}{c}{\textbf{SGCL}} & \multicolumn{2}{c}{\textbf{UGCL}} & \multicolumn{2}{c}{\textbf{BA-SGCL}} \\
\cmidrule(lr){3-4} \cmidrule(lr){5-6} \cmidrule(lr){7-8} \cmidrule(lr){9-10} \cmidrule(lr){11-12} \cmidrule(lr){13-14} \cmidrule(lr){15-16} \cmidrule(lr){17-18} \cmidrule(lr){19-20} \cmidrule(lr){21-22}
& & \multicolumn{1}{c}{AUC} & \multicolumn{1}{c}{Macro-F1} & \multicolumn{1}{c}{AUC} & \multicolumn{1}{c}{Macro-F1} & \multicolumn{1}{c}{AUC} & \multicolumn{1}{c}{Macro-F1} & \multicolumn{1}{c}{AUC} & \multicolumn{1}{c}{Macro-F1} & \multicolumn{1}{c}{AUC} & \multicolumn{1}{c}{Macro-F1} & \multicolumn{1}{c}{AUC} & \multicolumn{1}{c}{Macro-F1} & \multicolumn{1}{c}{AUC} & \multicolumn{1}{c}{Macro-F1} & \multicolumn{1}{c}{AUC} & \multicolumn{1}{c}{Macro-F1} & \multicolumn{1}{c}{AUC} & \multicolumn{1}{c}{Macro-F1} & \multicolumn{1}{c}{AUC} & \multicolumn{1}{c}{Macro-F1} \\
\midrule
\multirow{5}{*}{\rotatebox[origin=c]{90}{BitcoinAlpha}} 
& 0  & 0.8103 & 0.6718 & 0.7997 & 0.6652 & 0.8019 & 0.6728 & 0.8632 & 0.7102 & 0.8558 & 0.7142 & 0.8591 & 0.7208 & 0.8039 & 0.6847 & 0.8495 & 0.7125 & {\ul 0.8648} & {\ul 0.7357} & \textbf{0.8942} & \textbf{0.7778} \\
& 5  & 0.7418 & 0.5815 & 0.7358 & 0.5643 & 0.7425 & 0.5798 & 0.7978 & 0.6338 & 0.8034 & 0.6485 & 0.8007 & 0.6422 & 0.7459 & 0.5842 & 0.8078 & 0.6542 & {\ul 0.8262} & {\ul 0.6723} & \textbf{0.8465} & \textbf{0.6904} \\
& 10 & 0.6873 & 0.5217 & 0.6917 & 0.5128 & 0.6998 & 0.5235 & 0.7564 & 0.6042 & 0.7795 & 0.5748 & 0.7787 & 0.5684 & 0.7068 & 0.5327 & 0.7619 & 0.6052 & {\ul 0.7847} & {\ul 0.6222} & \textbf{0.7992} & \textbf{0.6523} \\
& 15 & 0.6712 & 0.4915 & 0.6852 & 0.4832 & 0.6917 & 0.4982 & 0.7324 & 0.5492 & 0.7462 & 0.5474 & 0.7435 & 0.5393 & 0.6194 & 0.5018 & 0.7252 & 0.5685 & {\ul 0.7482} & {\ul 0.5772} & \textbf{0.7712} & \textbf{0.5998} \\
& 20 & 0.6517 & 0.4758 & 0.6532 & 0.4708 & 0.6697 & 0.4885 & 0.6972 & 0.5134 & 0.7182 & 0.5124 & 0.6928 & 0.5074 & 0.6022 & 0.4832 & 0.6914 & 0.5317 & {\ul 0.7248} & {\ul 0.5422} & \textbf{0.7474} & \textbf{0.5747} \\
\cmidrule{1-22}
\multirow{5}{*}{\rotatebox[origin=c]{90}{BitcoinOTC}}
& 0  & 0.8215 & 0.7617 & 0.8257 & 0.7505 & 0.8308 & 0.7618 & 0.8854 & 0.7602 & {\ul 0.8967} & 0.7515 & 0.8842 & 0.7634 & 0.8175 & 0.7557 & 0.8935 & 0.7715 & 0.8947 & {\ul 0.7805} & \textbf{0.9104} & \textbf{0.8074} \\
& 5  & 0.7617 & 0.6617 & 0.7758 & 0.6475 & 0.7818 & 0.6518 & 0.8377 & 0.7232 & 0.8562 & 0.6924 & 0.8417 & 0.6842 & 0.7958 & 0.6578 & 0.8454 & 0.7214 & {\ul 0.8607} & {\ul 0.7554} & \textbf{0.8778} & \textbf{0.7818} \\
& 10 & 0.7498 & 0.6237 & 0.7508 & 0.6147 & 0.7515 & 0.6282 & 0.8034 & 0.6452 & 0.8262 & 0.6322 & 0.8108 & 0.6382 & 0.7457 & 0.5844 & 0.8095 & 0.6615 & {\ul 0.8324} & {\ul 0.6867} & \textbf{0.8472} & \textbf{0.6988} \\
& 15 & 0.7155 & 0.5794 & 0.7275 & 0.5818 & 0.7355 & 0.5938 & 0.7717 & 0.5918 & {\ul 0.7964} & 0.5828 & 0.7854 & 0.5742 & 0.6927 & 0.5482 & 0.7822 & 0.6095 & 0.7955 & {\ul 0.6407} & \textbf{0.8138} & \textbf{0.6452} \\
& 20 & 0.6817 & 0.5584 & 0.6982 & 0.5625 & 0.7042 & 0.5694 & 0.7415 & 0.5562 & 0.7565 & 0.5397 & 0.7302 & 0.5234 & 0.6607 & 0.5198 & 0.7402 & 0.5712 & {\ul 0.7658} & {\ul 0.5948} & \textbf{0.7922} & \textbf{0.6115} \\
\cmidrule{1-22}
\multirow{5}{*}{\rotatebox[origin=c]{90}{Slashdot}}
& 0  & 0.8218 & 0.6815 & 0.8156 & 0.6838 & 0.8293 & 0.6942 & 0.8384 & 0.7097 & 0.8904 & 0.7206 & {\ul 0.8937} & 0.7294 & 0.7824 & 0.6984 & 0.8844 & 0.6878 & 0.8882 & {\ul 0.7372} & \textbf{0.8953} & \textbf{0.7543} \\
& 5  & 0.7327 & 0.6324 & 0.7437 & 0.6337 & 0.7528 & 0.6477 & 0.7837 & 0.6877 & 0.8282 & 0.6862 & 0.8012 & 0.6778 & 0.7188 & 0.6488 & 0.8152 & 0.6502 & {\ul 0.8474} & {\ul 0.6944} & \textbf{0.8565} & \textbf{0.7482} \\
& 10 & 0.6917 & 0.5815 & 0.6897 & 0.5714 & 0.6994 & 0.5828 & 0.7627 & {\ul 0.6747} & 0.7695 & 0.6338 & 0.7402 & 0.6252 & 0.6568 & 0.5817 & 0.7462 & 0.5615 & {\ul 0.7775} & 0.6654 & \textbf{0.8012} & \textbf{0.7324} \\
& 15 & 0.6417 & 0.5417 & 0.6492 & 0.5408 & 0.6595 & 0.5567 & {\ul 0.7397} & {\ul 0.6434} & 0.7395 & 0.5984 & 0.7094 & 0.5822 & 0.6372 & 0.5552 & 0.6914 & 0.5004 & 0.7362 & 0.6392 & \textbf{0.7705} & \textbf{0.6847} \\
& 20 & 0.6215 & 0.5168 & 0.6344 & 0.5202 & 0.6412 & 0.5254 & {\ul 0.7157} & {\ul 0.6052} & 0.6977 & 0.5702 & 0.6702 & 0.5664 & 0.5972 & 0.5212 & 0.6584 & 0.5005 & 0.6915 & 0.5875 & \textbf{0.7647} & \textbf{0.6585} \\
\cmidrule{1-22}
\multirow{5}{*}{\rotatebox[origin=c]{90}{Epinions}}
& 0  & 0.7915 & 0.6842 & 0.7767 & 0.6952 & 0.7917 & 0.6994 & 0.8572 & 0.7108 & 0.8595 & 0.7145 & 0.8617 & 0.6788 & 0.7825 & {\ul 0.7165} & 0.8517 & 0.7152 & {\ul 0.8727} & 0.6865 & \textbf{0.8735} & \textbf{0.7305} \\
& 5  & 0.7812 & 0.6505 & 0.7715 & 0.6607 & 0.7837 & 0.6728 & 0.8078 & 0.6954 & 0.8265 & {\ul 0.7037} & 0.8057 & 0.6587 & 0.7532 & 0.6734 & 0.8038 & 0.6665 & {\ul 0.8357} & 0.6842 & \textbf{0.8527} & \textbf{0.7207} \\
& 10 & 0.7425 & 0.6118 & 0.7387 & 0.6122 & 0.7425 & 0.6235 & 0.7474 & 0.6602 & 0.7985 & {\ul 0.6894} & 0.7827 & 0.6362 & 0.7414 & 0.6437 & 0.7885 & 0.6532 & {\ul 0.8122} & 0.6784 & \textbf{0.8448} & \textbf{0.7325} \\
& 15 & 0.7277 & 0.5884 & 0.7147 & 0.5847 & 0.7152 & 0.5892 & 0.7207 & 0.6334 & 0.7817 & 0.6594 & 0.7657 & 0.6102 & 0.7252 & 0.6205 & 0.7447 & 0.6162 & {\ul 0.7875} & {\ul 0.6637} & \textbf{0.8032} & \textbf{0.7037} \\
& 20 & 0.6917 & 0.5717 & 0.6885 & 0.5602 & 0.6895 & 0.5674 & 0.6992 & 0.6072 & 0.7587 & 0.6362 & 0.7428 & 0.6072 & 0.6985 & 0.5967 & 0.7132 & 0.5877 & {\ul 0.7715} & {\ul 0.6462} & \textbf{0.7882} & \textbf{0.6777} \\
\bottomrule
\end{tabular}}
\label{balance-attack_performance_auc&macro}
\end{table*}

\begin{table*}[!t] \centering
\caption{AUC and Macro-F1 of SGNNs on link sign prediction under FlipAttack.}
\resizebox{\linewidth}{!}{
\begin{tabular}{c c cccccccccccccccccccc}
\toprule
\multirow{2}{*}{\textbf{Dataset}} & \multirow{2}{*}{\textbf{Ptb(\%)}} & \multicolumn{2}{c}{\textbf{SiNE}} & \multicolumn{2}{c}{\textbf{SGCN}} & \multicolumn{2}{c}{\textbf{SNEA}} & \multicolumn{2}{c}{\textbf{BESIDE}} & \multicolumn{2}{c}{\textbf{SDGNN}} & \multicolumn{2}{c}{\textbf{SDGCN}} & \multicolumn{2}{c}{\textbf{RSGNN}} & \multicolumn{2}{c}{\textbf{SGCL}} & \multicolumn{2}{c}{\textbf{UGCL}} & \multicolumn{2}{c}{\textbf{BA-SGCL}} \\
\cmidrule(lr){3-4} \cmidrule(lr){5-6} \cmidrule(lr){7-8} \cmidrule(lr){9-10} \cmidrule(lr){11-12} \cmidrule(lr){13-14} \cmidrule(lr){15-16} \cmidrule(lr){17-18} \cmidrule(lr){19-20} \cmidrule(lr){21-22}
& & \multicolumn{1}{c}{AUC} & \multicolumn{1}{c}{Macro-F1} & \multicolumn{1}{c}{AUC} & \multicolumn{1}{c}{Macro-F1} & \multicolumn{1}{c}{AUC} & \multicolumn{1}{c}{Macro-F1} & \multicolumn{1}{c}{AUC} & \multicolumn{1}{c}{Macro-F1} & \multicolumn{1}{c}{AUC} & \multicolumn{1}{c}{Macro-F1} & \multicolumn{1}{c}{AUC} & \multicolumn{1}{c}{Macro-F1} & \multicolumn{1}{c}{AUC} & \multicolumn{1}{c}{Macro-F1} & \multicolumn{1}{c}{AUC} & \multicolumn{1}{c}{Macro-F1} & \multicolumn{1}{c}{AUC} & \multicolumn{1}{c}{Macro-F1} & \multicolumn{1}{c}{AUC} & \multicolumn{1}{c}{Macro-F1} \\
\midrule
\multirow{5}{*}{\rotatebox[origin=c]{90}{BitcoinAlpha}} 
& 0  & 0.8103 & 0.6718 & 0.7997 & 0.6652 & 0.8019 & 0.6728 & 0.8632 & 0.7102 & 0.8558 & 0.7142 & 0.8591 & 0.7208 & 0.8039 & 0.6847 & 0.8496 & 0.7125 & {\ul 0.8648} & {\ul 0.7357} & \textbf{0.8942} & \textbf{0.7778} \\
& 5  & 0.7133 & 0.5614 & 0.7118 & 0.5787 & 0.7226 & 0.5897 & 0.7883 & 0.6542 & 0.7482 & 0.6315 & 0.7425 & 0.6319 & 0.7304 & 0.5843 & 0.7745 & 0.6324 & {\ul 0.7997} & {\ul 0.6592} & \textbf{0.8208} & \textbf{0.6705} \\
& 10 & 0.6858 & 0.5768 & 0.6743 & 0.5547 & 0.6882 & 0.5626 & 0.7447 & {\ul 0.5982} & 0.7256 & 0.5772 & 0.7398 & 0.5785 & 0.6902 & 0.5438 & 0.7206 & 0.5573 & {\ul 0.7563} & 0.5978 & \textbf{0.7824} & \textbf{0.6248} \\
& 15 & 0.5917 & 0.4928 & 0.5988 & 0.4983 & 0.6018 & 0.5034 & 0.6657 & {\ul 0.5728} & 0.6619 & 0.5637 & 0.6634 & 0.5687 & 0.6267 & 0.5054 & 0.6594 & 0.5134 & {\ul 0.7087} & 0.5677 & \textbf{0.7293} & \textbf{0.5907} \\
& 20 & 0.5718 & 0.4416 & 0.5607 & 0.4327 & 0.5776 & 0.4428 & {\ul 0.6408} & {\ul 0.5193} & 0.6322 & 0.5084 & 0.6254 & 0.5087 & 0.5786 & 0.4578 & 0.5884 & 0.4807 & 0.6297 & 0.5043 & \textbf{0.6742} & \textbf{0.5392} \\
\cmidrule{1-22}
\multirow{5}{*}{\rotatebox[origin=c]{90}{BitcoinOTC}}
& 0  & 0.8215 & 0.7617 & 0.8257 & 0.7505 & 0.8308 & 0.7618 & 0.8854 & 0.7602 & {\ul 0.8967} & 0.7515 & 0.8842 & 0.7634 & 0.8175 & 0.7557 & 0.8935 & 0.7715 & 0.8947 & {\ul 0.7805} & \textbf{0.9104} & \textbf{0.8074} \\
& 5  & 0.7305 & 0.6438 & 0.7394 & 0.6534 & 0.7414 & 0.6638 & {\ul 0.7875} & 0.6872 & 0.7743 & 0.6752 & 0.7828 & 0.6797 & 0.7443 & 0.6795 & 0.7725 & {\ul 0.6882} & 0.7833 & 0.6844 & \textbf{0.8002} & \textbf{0.7053} \\
& 10 & 0.6902 & 0.5993 & 0.6848 & 0.6083 & 0.6928 & 0.6138 & 0.7063 & 0.6113 & 0.7083 & 0.5984 & 0.7126 & 0.5986 & 0.7016 & 0.6274 & 0.6813 & 0.6018 & {\ul 0.7236} & {\ul 0.6294} & \textbf{0.7447} & \textbf{0.6633} \\
& 15 & 0.6594 & 0.5593 & 0.6514 & 0.5634 & 0.6638 & 0.5694 & 0.6957 & 0.5914 & 0.6884 & 0.5845 & 0.6945 & 0.5917 & 0.6745 & 0.5853 & 0.6643 & 0.5892 & {\ul 0.7113} & {\ul 0.5983} & \textbf{0.7416} & \textbf{0.6337} \\
& 20 & 0.6428 & 0.5583 & 0.6312 & 0.5517 & 0.6417 & 0.5638 & 0.6504 & {\ul 0.5827} & 0.6718 & 0.5816 & 0.6734 & 0.5825 & 0.6673 & 0.5737 & 0.6434 & 0.5624 & {\ul 0.7033} & 0.5767 & \textbf{0.7223} & \textbf{0.6273} \\
\cmidrule{1-22}
\multirow{5}{*}{\rotatebox[origin=c]{90}{Slashdot}}
& 0  & 0.8218 & 0.6815 & 0.8156 & 0.6838 & 0.8293 & 0.6942 & 0.8384 & 0.7097 & 0.8904 & 0.7206 & {\ul 0.8937} & 0.7294 & 0.7824 & 0.6984 & 0.8844 & 0.6878 & 0.8882 & {\ul 0.7372} & \textbf{0.8953} & \textbf{0.7543} \\
& 5  & 0.7225 & 0.6028 & 0.7196 & 0.6097 & 0.7282 & 0.6128 & 0.7817 & 0.6353 & 0.7692 & 0.6238 & 0.7588 & 0.6176 & 0.7256 & 0.6114 & 0.7732 & 0.6305 & {\ul 0.7868} & {\ul 0.6425} & \textbf{0.8018} & \textbf{0.6627} \\
& 10 & 0.6394 & 0.5274 & 0.6415 & 0.5217 & 0.6596 & 0.5335 & 0.7392 & 0.5684 & 0.7197 & 0.5584 & 0.7154 & 0.5494 & 0.6604 & 0.5326 & 0.7218 & 0.5662 & {\ul 0.7452} & {\ul 0.5815} & \textbf{0.7623} & \textbf{0.6098} \\
& 15 & 0.6192 & 0.5023 & 0.6124 & 0.5054 & 0.6242 & 0.5282 & 0.6918 & 0.5585 & 0.6814 & 0.5456 & 0.6842 & 0.5415 & 0.6278 & 0.5115 & 0.6742 & 0.5416 & {\ul 0.6974} & {\ul 0.5637} & \textbf{0.7238} & \textbf{0.5753} \\
& 20 & 0.6128 & 0.4685 & 0.6045 & 0.4773 & 0.6187 & 0.4876 & 0.6722 & 0.5196 & 0.6514 & 0.4984 & 0.6533 & 0.5024 & 0.6232 & 0.4894 & 0.6593 & 0.5022 & {\ul 0.6807} & {\ul 0.5238} & \textbf{0.7032} & \textbf{0.5417} \\
\cmidrule{1-22}
\multirow{5}{*}{\rotatebox[origin=c]{90}{Epinions}}
& 0  & 0.7915 & 0.6842 & 0.7767 & 0.6952 & 0.7917 & 0.6994 & 0.8572 & 0.7108 & 0.8595 & 0.7145 & 0.8617 & 0.6788 & 0.7825 & {\ul 0.7165} & 0.8517 & 0.7152 & {\ul 0.8727} & 0.6865 & \textbf{0.8735} & \textbf{0.7305} \\
& 5  & 0.7116 & 0.5784 & 0.7128 & 0.5828 & 0.7293 & 0.5986 & 0.7768 & 0.6383 & 0.7433 & 0.6354 & 0.7366 & 0.6385 & 0.7383 & 0.6125 & 0.7762 & {\ul 0.6456} & {\ul 0.7884} & 0.6415 & \textbf{0.7963} & \textbf{0.6558} \\
& 10 & 0.6818 & 0.5164 & 0.6894 & 0.5212 & 0.6938 & 0.5356 & 0.7218 & 0.5882 & {\ul 0.7352} & 0.5716 & 0.7292 & 0.5875 & 0.7143 & 0.5528 & 0.7213 & 0.5815 & 0.7315 & {\ul 0.5894} & \textbf{0.7568} & \textbf{0.6014} \\
& 15 & 0.6595 & 0.4914 & 0.6694 & 0.4984 & 0.6778 & 0.5037 & 0.6938 & 0.5412 & {\ul 0.7048} & 0.5334 & 0.6982 & 0.5384 & 0.6825 & 0.5126 & 0.6973 & 0.5413 & 0.7015 & {\ul 0.5546} & \textbf{0.7214} & \textbf{0.5714} \\
& 20 & 0.6083 & 0.4993 & 0.6144 & 0.4888 & 0.6252 & 0.4986 & 0.6615 & 0.5264 & 0.6704 & 0.5274 & 0.6688 & 0.5283 & 0.6432 & 0.5056 & 0.6594 & 0.5335 & {\ul 0.6788} & {\ul 0.5416} & \textbf{0.6946} & \textbf{0.5615} \\
\bottomrule
\end{tabular}}
\label{flipattack_performance_auc&macro}
\end{table*}

\begin{table*}[!h] \centering
\caption{Micro-F1 and Binary-F1 of SGNNs on link sign prediction under balance-attack.}
\resizebox{\linewidth}{!}{
\begin{tabular}{c c cccccccccccccccccccc}
\toprule
\multirow{2}{*}{\textbf{Dataset}} & \multirow{2}{*}{\textbf{Ptb(\%)}} & \multicolumn{2}{c}{\textbf{SiNE}} & \multicolumn{2}{c}{\textbf{SGCN}} & \multicolumn{2}{c}{\textbf{SNEA}} & \multicolumn{2}{c}{\textbf{BESIDE}} & \multicolumn{2}{c}{\textbf{SDGNN}} & \multicolumn{2}{c}{\textbf{SDGCN}} & \multicolumn{2}{c}{\textbf{RSGNN}} & \multicolumn{2}{c}{\textbf{SGCL}} & \multicolumn{2}{c}{\textbf{UGCL}} & \multicolumn{2}{c}{\textbf{BA-SGCL}} \\
\cmidrule(lr){3-4} \cmidrule(lr){5-6} \cmidrule(lr){7-8} \cmidrule(lr){9-10} \cmidrule(lr){11-12} \cmidrule(lr){13-14} \cmidrule(lr){15-16} \cmidrule(lr){17-18} \cmidrule(lr){19-20} \cmidrule(lr){21-22}
& & \multicolumn{1}{c}{Micro-F1} & \multicolumn{1}{c}{Binary-F1} & \multicolumn{1}{c}{Micro-F1} & \multicolumn{1}{c}{Binary-F1} & \multicolumn{1}{c}{Micro-F1} & \multicolumn{1}{c}{Binary-F1} & \multicolumn{1}{c}{Micro-F1} & \multicolumn{1}{c}{Binary-F1} & \multicolumn{1}{c}{Micro-F1} & \multicolumn{1}{c}{Binary-F1} & \multicolumn{1}{c}{Micro-F1} & \multicolumn{1}{c}{Binary-F1} & \multicolumn{1}{c}{Micro-F1} & \multicolumn{1}{c}{Binary-F1} & \multicolumn{1}{c}{Micro-F1} & \multicolumn{1}{c}{Binary-F1} & \multicolumn{1}{c}{Micro-F1} & \multicolumn{1}{c}{Binary-F1} & \multicolumn{1}{c}{Micro-F1} & \multicolumn{1}{c}{Binary-F1} \\
\midrule
\multirow{5}{*}{\rotatebox[origin=c]{90}{BitcoinAlpha}} 
& 0  & 0.8497 & 0.8956 & 0.8523 & 0.9044 & 0.8615 & 0.9185 & 0.9293 & 0.9494 & 0.9116 & 0.9445 & 0.9214 & 0.9472 & 0.8824 & 0.9343 & 0.9234 & 0.9593 & {\ul 0.9484} & {\ul 0.9734} & \textbf{0.9528} & \textbf{0.9753} \\
& 5  & 0.7135 & 0.8243 & 0.7265 & 0.8137 & 0.7322 & 0.8273 & 0.8957 & 0.9213 & 0.8756 & 0.9244 & 0.8737 & 0.9278 & 0.7565 & 0.8536 & 0.9162 & 0.9355 & {\ul 0.9212} & {\ul 0.9576} & \textbf{0.9294} & \textbf{0.9623} \\
& 10 & 0.6595 & 0.7577 & 0.6604 & 0.7682 & 0.6734 & 0.7755 & 0.8582 & 0.9146 & 0.8305 & 0.8984 & 0.8217 & 0.8892 & 0.6805 & 0.7987 & {\ul 0.8845} & 0.9172 & 0.8835 & {\ul 0.9367} & \textbf{0.9034} & \textbf{0.9458} \\
& 15 & 0.6157 & 0.7482 & 0.6202 & 0.7565 & 0.6305 & 0.7636 & 0.8085 & 0.8983 & 0.7964 & 0.8715 & 0.7885 & 0.8695 & 0.6602 & 0.7865 & \textbf{0.8562} & 0.9005 & 0.8455 & {\ul 0.9132} & {\ul 0.8432} & \textbf{0.9163} \\
& 20 & 0.5985 & 0.7293 & 0.6025 & 0.7343 & 0.6134 & 0.7484 & 0.7632 & {\ul 0.8794} & 0.7372 & 0.8395 & 0.7325 & 0.8272 & 0.6305 & 0.7634 & {\ul 0.8054} & 0.8734 & 0.7995 & 0.8754 & \textbf{0.8085} & \textbf{0.8805} \\
\cmidrule{1-22}
\multirow{5}{*}{\rotatebox[origin=c]{90}{BitcoinOTC}}
& 0  & 0.8684 & 0.9182 & 0.8794 & 0.9293 & 0.8827 & 0.9377 & 0.9125 & 0.9457 & 0.9047 & 0.9435 & 0.9124 & 0.9394 & 0.8915 & 0.9385 & 0.9205 & 0.9567 & {\ul 0.9364} & {\ul 0.9654} & \textbf{0.9385} & \textbf{0.9665} \\
& 5  & 0.7735 & 0.8523 & 0.7825 & 0.8664 & 0.7955 & 0.8772 & 0.8842 & 0.9354 & 0.8674 & 0.9227 & 0.8645 & 0.9273 & 0.7974 & 0.8764 & 0.9077 & 0.9444 & {\ul 0.9153} & {\ul 0.9534} & \textbf{0.9227} & \textbf{0.9574} \\
& 10 & 0.7315 & 0.8304 & 0.7477 & 0.8405 & 0.7535 & 0.8534 & 0.8265 & 0.9043 & 0.8145 & 0.8894 & 0.8045 & 0.8846 & 0.7137 & 0.8155 & 0.8545 & {\ul 0.9197} & \textbf{0.8665} & \textbf{0.9235} & {\ul 0.8593} & 0.9185 \\
& 15 & 0.6943 & 0.8094 & 0.7063 & 0.8107 & 0.7155 & 0.8255 & 0.7875 & 0.8712 & 0.7654 & 0.8565 & 0.7665 & 0.8576 & 0.6784 & 0.7905 & {\ul 0.8132} & 0.8943 & \textbf{0.8237} & \textbf{0.8967} & 0.8137 & {\ul 0.8893} \\
& 20 & 0.6727 & 0.7843 & 0.6836 & 0.7922 & 0.6985 & 0.8097 & 0.7405 & 0.8424 & 0.7085 & 0.8174 & 0.7172 & 0.8272 & 0.6427 & 0.7622 & 0.7734 & \textbf{0.8662} & \textbf{0.7755} & {\ul 0.8645} & {\ul 0.7746} & 0.8633 \\
\cmidrule{1-22}
\multirow{5}{*}{\rotatebox[origin=c]{90}{Slashdot}}
& 0  & 0.8015 & 0.8617 & 0.8124 & 0.8785 & 0.8232 & 0.8856 & 0.8545 & 0.9142 & 0.8695 & 0.9292 & 0.8645 & 0.9272 & 0.7826 & 0.8577 & 0.8735 & 0.9294 & {\ul 0.8794} & {\ul 0.9294} & \textbf{0.8795} & \textbf{0.9304} \\
& 5  & 0.7012 & 0.7985 & 0.7195 & 0.8094 & 0.7203 & 0.8134 & 0.8367 & 0.8785 & 0.8575 & 0.9164 & 0.8423 & 0.9083 & 0.7447 & 0.8222 & 0.8284 & 0.8992 & {\ul 0.8735} & \textbf{0.9275} & \textbf{0.8744} & {\ul 0.9265} \\
& 10 & 0.6415 & 0.7464 & 0.6515 & 0.7525 & 0.6645 & 0.7645 & 0.8097 & 0.8515 & 0.8053 & 0.8815 & 0.7987 & 0.8772 & 0.6715 & 0.7764 & 0.7575 & 0.8556 & {\ul 0.8536} & \textbf{0.9153} & \textbf{0.8705} & {\ul 0.9106} \\
& 15 & 0.6015 & 0.6917 & 0.6022 & 0.7096 & 0.6193 & 0.7186 & 0.7835 & 0.8212 & 0.7707 & 0.8565 & 0.7634 & 0.8425 & 0.6375 & 0.7463 & 0.7224 & 0.8334 & {\ul 0.8304} & {\ul 0.9014} & \textbf{0.8394} & \textbf{0.9057} \\
& 20 & 0.5855 & 0.6914 & 0.5935 & 0.7025 & 0.6065 & 0.7135 & 0.7423 & 0.8017 & 0.7323 & 0.8283 & 0.7283 & 0.8115 & 0.6005 & 0.7162 & 0.7005 & 0.8166 & \textbf{0.7823} & \textbf{0.8707} & {\ul 0.7814} & {\ul 0.8634} \\
\cmidrule{1-22}
\multirow{5}{*}{\rotatebox[origin=c]{90}{Epinions}}
& 0  & 0.8092 & 0.8727 & 0.8184 & 0.8866 & 0.8282 & 0.8967 & 0.8545 & 0.9156 & 0.8724 & 0.9262 & 0.8634 & 0.9255 & 0.8284 & 0.8935 & 0.8676 & 0.9235 & \textbf{0.8765} & \textbf{0.9307} & {\ul 0.8745} & {\ul 0.9275} \\
& 5  & 0.7627 & 0.8445 & 0.7737 & 0.8525 & 0.7845 & 0.8634 & 0.8363 & 0.9002 & 0.8615 & 0.9193 & 0.8536 & 0.9077 & 0.7733 & 0.8545 & 0.8475 & 0.9122 & \textbf{0.8723} & \textbf{0.9284} & {\ul 0.8715} & {\ul 0.9275} \\
& 10 & 0.6917 & 0.7874 & 0.7035 & 0.7973 & 0.7115 & 0.8097 & 0.7935 & 0.8817 & 0.8377 & 0.9035 & 0.8225 & 0.9057 & 0.7345 & 0.8237 & 0.8215 & 0.8954 & \textbf{0.8675} & \textbf{0.9245} & {\ul 0.8605} & {\ul 0.9184} \\
& 15 & 0.6613 & 0.7685 & 0.6714 & 0.7704 & 0.6835 & 0.7846 & 0.7692 & 0.8415 & 0.8094 & 0.8834 & 0.7977 & 0.8857 & 0.7065 & 0.8013 & 0.7755 & 0.8633 & \textbf{0.8535} & \textbf{0.9155} & {\ul 0.8454} & {\ul 0.9087} \\
& 20 & 0.6377 & 0.7446 & 0.6456 & 0.7494 & 0.6573 & 0.7583 & 0.7335 & 0.8144 & 0.7693 & 0.8554 & 0.7524 & 0.8493 & 0.6835 & 0.7833 & 0.7382 & 0.8374 & \textbf{0.8325} & \textbf{0.9015} & {\ul 0.8215} & {\ul 0.8935} \\
\bottomrule
\end{tabular}}
\label{balance-attack_performance_micro&binary}
\end{table*}

\begin{table*}[!h] \centering
\caption{Micro-F1 and Binary-F1 of SGNNs on link sign prediction under FlipAttack.}
\resizebox{\linewidth}{!}{
\begin{tabular}{c c cccccccccccccccccccc}
\toprule
\multirow{2}{*}{\textbf{Dataset}} & \multirow{2}{*}{\textbf{Ptb(\%)}} & \multicolumn{2}{c}{\textbf{SiNE}} & \multicolumn{2}{c}{\textbf{SGCN}} & \multicolumn{2}{c}{\textbf{SNEA}} & \multicolumn{2}{c}{\textbf{BESIDE}} & \multicolumn{2}{c}{\textbf{SDGNN}} & \multicolumn{2}{c}{\textbf{SDGCN}} & \multicolumn{2}{c}{\textbf{RSGNN}} & \multicolumn{2}{c}{\textbf{SGCL}} & \multicolumn{2}{c}{\textbf{UGCL}} & \multicolumn{2}{c}{\textbf{BA-SGCL}} \\
\cmidrule(lr){3-4} \cmidrule(lr){5-6} \cmidrule(lr){7-8} \cmidrule(lr){9-10} \cmidrule(lr){11-12} \cmidrule(lr){13-14} \cmidrule(lr){15-16} \cmidrule(lr){17-18} \cmidrule(lr){19-20} \cmidrule(lr){21-22}
& & \multicolumn{1}{c}{Micro-F1} & \multicolumn{1}{c}{Binary-F1} & \multicolumn{1}{c}{Micro-F1} & \multicolumn{1}{c}{Binary-F1} & \multicolumn{1}{c}{Micro-F1} & \multicolumn{1}{c}{Binary-F1} & \multicolumn{1}{c}{Micro-F1} & \multicolumn{1}{c}{Binary-F1} & \multicolumn{1}{c}{Micro-F1} & \multicolumn{1}{c}{Binary-F1} & \multicolumn{1}{c}{Micro-F1} & \multicolumn{1}{c}{Binary-F1} & \multicolumn{1}{c}{Micro-F1} & \multicolumn{1}{c}{Binary-F1} & \multicolumn{1}{c}{Micro-F1} & \multicolumn{1}{c}{Binary-F1} & \multicolumn{1}{c}{Micro-F1} & \multicolumn{1}{c}{Binary-F1} & \multicolumn{1}{c}{Micro-F1} & \multicolumn{1}{c}{Binary-F1} \\
\midrule
\multirow{5}{*}{\rotatebox[origin=c]{90}{BitcoinAlpha}}
& 0  & 0.8497 & 0.8956 & 0.8523 & 0.9044 & 0.8615 & 0.9185 & 0.9293 & 0.9494 & 0.9116 & 0.9445 & 0.9214 & 0.9472 & 0.8824 & 0.9343 & 0.9234 & 0.9593 & {\ul 0.9484} & {\ul 0.9734} & \textbf{0.9528} & \textbf{0.9753} \\
& 5  & 0.7291 & 0.8347 & 0.7254 & 0.8298 & 0.7318 & 0.8382 & 0.8538 & 0.9128 & 0.8335 & 0.9026 & 0.8311 & 0.9127 & 0.7506 & 0.8464 & 0.8671 & 0.9161 & \textbf{0.8813} & \textbf{0.9261} & {\ul 0.8745} & {\ul 0.9217} \\
& 10 & 0.6728 & 0.8027 & 0.6742 & 0.7987 & 0.6881 & 0.8048 & 0.8281 & 0.8797 & 0.8177 & 0.8596 & 0.8047 & 0.8754 & 0.6997 & 0.8107 & 0.8205 & 0.9042 & \textbf{0.8594} & \textbf{0.9226} & {\ul 0.8573} & {\ul 0.9201} \\
& 15 & 0.6394 & 0.7725 & 0.6311 & 0.7637 & 0.6442 & 0.7731 & 0.7831 & 0.8628 & 0.7662 & 0.8421 & 0.7527 & 0.8311 & 0.6641 & 0.7854 & 0.7702 & 0.8737 & \textbf{0.8163} & \textbf{0.8998} & {\ul 0.8157} & {\ul 0.8948} \\
& 20 & 0.5728 & 0.7128 & 0.5794 & 0.7078 & 0.5877 & 0.7268 & 0.7383 & 0.8311 & 0.7127 & 0.8128 & 0.7041 & 0.8097 & 0.5904 & 0.7263 & 0.6917 & 0.8187 & {\ul 0.7454} & {\ul 0.8343} & \textbf{0.7454} & \textbf{0.8364} \\
\cmidrule{1-22}
\multirow{5}{*}{\rotatebox[origin=c]{90}{BitcoinOTC}}
& 0  & 0.8684 & 0.9182 & 0.8794 & 0.9293 & 0.8827 & 0.9377 & 0.9125 & 0.9457 & 0.9047 & 0.9435 & 0.9124 & 0.9394 & 0.8915 & 0.9385 & 0.9205 & 0.9567 & {\ul 0.9364} & {\ul 0.9654} & \textbf{0.9385} & \textbf{0.9665} \\
& 5  & 0.7688 & 0.8587 & 0.7783 & 0.8618 & 0.7981 & 0.8773 & 0.8228 & 0.9033 & 0.8128 & 0.9191 & 0.8064 & 0.9117 & 0.8004 & 0.8763 & 0.8413 & 0.9037 & {\ul 0.8501} & {\ul 0.9117} & \textbf{0.8605} & \textbf{0.9184} \\
& 10 & 0.7248 & 0.8274 & 0.7234 & 0.8257 & 0.7478 & 0.8364 & 0.8127 & 0.8783 & 0.8017 & 0.8828 & 0.7994 & 0.8771 & 0.7507 & 0.8417 & 0.7934 & 0.8847 & {\ul 0.8488} & {\ul 0.9044} & \textbf{0.8577} & \textbf{0.9118} \\
& 15 & 0.6687 & 0.7744 & 0.6717 & 0.7768 & 0.6831 & 0.7898 & 0.7781 & 0.8644 & 0.7667 & 0.8531 & 0.7528 & 0.8434 & 0.6967 & 0.7994 & 0.7795 & 0.8728 & {\ul 0.8148} & \textbf{0.8884} & \textbf{0.8185} & {\ul 0.8868} \\
& 20 & 0.6484 & 0.7516 & 0.6421 & 0.7593 & 0.6557 & 0.7623 & 0.7647 & 0.8323 & 0.7468 & 0.8255 & 0.7358 & 0.8044 & 0.6785 & 0.7844 & 0.7344 & 0.8421 & {\ul 0.8033} & {\ul 0.8816} & \textbf{0.8077} & \textbf{0.8881} \\
\cmidrule{1-22}
\multirow{5}{*}{\rotatebox[origin=c]{90}{Slashdot}}
& 0  & 0.8015 & 0.8617 & 0.8124 & 0.8785 & 0.8232 & 0.8856 & 0.8545 & 0.9142 & 0.8695 & 0.9292 & 0.8645 & 0.9272 & 0.7826 & 0.8577 & 0.8735 & 0.9294 & {\ul 0.8794} & {\ul 0.9294} & \textbf{0.8795} & \textbf{0.9304} \\
& 5  & 0.7144 & 0.7428 & 0.7083 & 0.7596 & 0.7183 & 0.7623 & 0.7743 & 0.8613 & 0.7794 & 0.8633 & 0.7681 & 0.8544 & 0.7107 & 0.7793 & 0.8038 & 0.8553 & \textbf{0.8247} & {\ul 0.8793} & {\ul 0.8203} & \textbf{0.8846} \\
& 10 & 0.6733 & 0.7414 & 0.6863 & 0.7471 & 0.6947 & 0.7587 & 0.7714 & 0.8516 & 0.7623 & 0.8533 & 0.7531 & 0.8383 & 0.7027 & 0.7631 & 0.7946 & 0.8525 & \textbf{0.8118} & {\ul 0.8713} & {\ul 0.8028} & \textbf{0.8733} \\
& 15 & 0.6691 & 0.7283 & 0.6646 & 0.7324 & 0.6704 & 0.7424 & 0.7534 & 0.8374 & 0.7424 & 0.8394 & 0.7384 & 0.8185 & 0.6883 & 0.7536 & 0.7736 & 0.8324 & {\ul 0.7804} & {\ul 0.8461} & \textbf{0.7881} & \textbf{0.8577} \\
& 20 & 0.6268 & 0.7014 & 0.6283 & 0.7096 & 0.6374 & 0.7216 & 0.7278 & 0.8096 & 0.7214 & 0.8213 & 0.7263 & 0.8003 & 0.6536 & 0.7324 & 0.7447 & 0.8196 & {\ul 0.7578} & {\ul 0.8325} & \textbf{0.7633} & \textbf{0.8383} \\
\cmidrule{1-22}
\multirow{5}{*}{\rotatebox[origin=c]{90}{Epinions}}
& 0  & 0.8092 & 0.8727 & 0.8184 & 0.8866 & 0.8282 & 0.8967 & 0.8545 & 0.9156 & 0.8724 & 0.9262 & 0.8634 & 0.9255 & 0.8284 & 0.8935 & 0.8676 & 0.9235 & \textbf{0.8765} & \textbf{0.9307} & {\ul 0.8745} & {\ul 0.9275} \\
& 5  & 0.7147 & 0.7711 & 0.7121 & 0.7764 & 0.7248 & 0.7883 & 0.8184 & 0.8633 & 0.8147 & 0.8666 & 0.8027 & 0.8624 & 0.7383 & 0.7954 & 0.8238 & 0.8631 & {\ul 0.8314} & \textbf{0.8983} & \textbf{0.8356} & {\ul 0.8971} \\
& 10 & 0.7013 & 0.7623 & 0.7094 & 0.7631 & 0.7216 & 0.7754 & 0.7986 & 0.8525 & 0.8027 & 0.8515 & 0.7986 & 0.8473 & 0.7324 & 0.7833 & 0.8023 & 0.8524 & {\ul 0.8274} & \textbf{0.8943} & \textbf{0.8327} & {\ul 0.8901} \\
& 15 & 0.6931 & 0.7374 & 0.7027 & 0.7445 & 0.7141 & 0.7524 & 0.7984 & 0.8318 & 0.7997 & 0.8396 & 0.7924 & 0.8246 & 0.7144 & 0.7644 & 0.8011 & 0.8325 & \textbf{0.8257} & {\ul 0.8723} & {\ul 0.8217} & \textbf{0.8778} \\
& 20 & 0.6884 & 0.7214 & 0.6831 & 0.7284 & 0.6986 & 0.7371 & 0.7827 & 0.8295 & 0.7727 & 0.8276 & 0.7631 & 0.8196 & 0.6987 & 0.7453 & 0.7868 & 0.8261 & {\ul 0.7996} & {\ul 0.8371} & \textbf{0.8027} & \textbf{0.8436} \\
\bottomrule
\end{tabular}}
\label{flipattack_performance_micro&binary}
\end{table*}

We evaluate model performance under two global attack scenarios with perturbation ratios ranging from $0\%$ to $20\%$. The results, presented in Tables \ref{balance-attack_performance_auc&macro}-\ref{flipattack_performance_micro&binary}, lead to the following key observations:

First, we observe a distinct performance pattern across the different evaluation metrics. BA-SGCL achieves substantial improvements in AUC and Macro-F1 over all baselines, while its gains in Micro-F1 and Binary-F1 appear more moderate. This discrepancy is explained by the inherent class imbalance in signed graphs. To provide direct evidence of our model's superiority, we analyze the composition of misclassified edges. As shown in Table~\ref{proportion_positive_edges}, BA-SGCL consistently misclassifies a lower proportion of positive edges than the baselines. This enhanced predictive capability stems directly from our core contrastive objective: by forcing the encoder to align representations from the low-balance poisoned graph with those from a high-balance augmented view, BA-SGCL learns embeddings that are inherently robust and retain the characteristics of a well-structured graph, leading to more accurate predictions.

Second, existing SGNNs degrade significantly under both attacks, while our model maintains high performance with minimal degradation. This is particularly evident when comparing against RSGNN, which is vulnerable to these targeted attacks despite being designed for robustness. This supports our claim that its direct defense mechanism—attempting to restore graph balance—is susceptible to the ``Irreversibility of Balance-related Information'' challenge. In contrast, BA-SGCL does not attempt to directly recover the flawed graph. By learning robust features in the latent space through contrastive learning, our model effectively bypasses the Irreversibility challenge that plagues direct-recovery methods.

Third, BA-SGCL outperforms other GCL-based models (SGCL, UGCL) even on clean graphs (i.e., with a $0\%$ perturbation rate). This highlights a key advantage of our approach. By contrasting against a meaningfully structured, high-balance positive view, the encoder learns more fundamental properties of signed graphs than methods that contrast against randomly perturbed views. Our guided balance augmentation provides a stronger and more relevant learning signal, resulting in higher-quality node embeddings for the downstream task even in the absence of an attack.

\begin{table}[!h] \centering
\caption{Proportion of Positive Edges in Misclassified Samples under balance-attack (perturbation rate = $20\%$). Lower values indicate better positive edge prediction performance.}
\begin{tabular}{ccc}
\toprule
\textbf{Dataset} & \textbf{UGCL} & \textbf{BA-SGCL} \\
\midrule
BitcoinAlpha & 60.08\% & \textbf{42.25\%} \\
BitcoinOTC   & 68.64\% & \textbf{54.38\%} \\
Slashdot     & 36.24\% & \textbf{18.56\%} \\
Epinions     & 24.03\% & \textbf{15.07\%} \\
\bottomrule
\end{tabular}
\label{proportion_positive_edges}
\end{table}

\begin{table}[!h] \centering
\caption{Effectiveness of balance augmentation under balance-attack.}
\resizebox{\linewidth}{!}{
\begin{tabular}{c c cccc cccc}
\toprule
\multirow{2}{*}{\textbf{Dataset}} & \multirow{2}{*}{\textbf{Ptb(\%)}} & \multicolumn{4}{c}{\textbf{random-SGCL}} & \multicolumn{4}{c}{\textbf{BA-SGCL}} \\
\cmidrule(lr){3-6} \cmidrule(lr){7-10}
& & \multicolumn{1}{c}{AUC} & \multicolumn{1}{c}{Macro-F1} & \multicolumn{1}{c}{Micro-F1} & \multicolumn{1}{c}{Binary-F1} & \multicolumn{1}{c}{AUC} & \multicolumn{1}{c}{Macro-F1} & \multicolumn{1}{c}{Micro-F1} & \multicolumn{1}{c}{Binary-F1} \\
\midrule
\multirow{3}{*}{BitcoinAlpha} & 0  & 0.8363 & 0.7258 & 0.9221 & 0.9601 & \textbf{0.8942} & \textbf{0.7778} & \textbf{0.9528} & \textbf{0.9753} \\
                              & 10 & 0.7692 & 0.6123 & 0.8778 & 0.9103 & \textbf{0.7992} & \textbf{0.6523} & \textbf{0.9034} & \textbf{0.9458} \\
                              & 20 & 0.7034 & 0.5401 & 0.8005 & 0.8762 & \textbf{0.7474} & \textbf{0.5747} & \textbf{0.8085} & \textbf{0.8805} \\
\cmidrule{1-10}
\multirow{3}{*}{BitcoinOTC} & 0  & 0.8894 & 0.7723 & 0.9189 & 0.9587 & \textbf{0.9104} & \textbf{0.8074} & \textbf{0.9385} & \textbf{0.9665} \\
                            & 10 & 0.8023 & 0.6632 & 0.8541 & 0.9083 & \textbf{0.8472} & \textbf{0.6988} & \textbf{0.8593} & \textbf{0.9185} \\
                            & 20 & 0.7539 & 0.5885 & 0.7724 & 0.8623 & \textbf{0.7922} & \textbf{0.6115} & \textbf{0.7746} & \textbf{0.8633} \\
\cmidrule{1-10}
\multirow{3}{*}{Slashdot}   & 0  & 0.8814 & 0.6889 & 0.8746 & 0.9223 & \textbf{0.8953} & \textbf{0.7543} & \textbf{0.8795} & \textbf{0.9304} \\
                            & 10 & 0.7498 & 0.5789 & 0.7982 & 0.8779 & \textbf{0.8012} & \textbf{0.7324} & \textbf{0.8705} & \textbf{0.9106} \\
                            & 20 & 0.6798 & 0.5122 & 0.7652 & 0.8523 & \textbf{0.7647} & \textbf{0.6585} & \textbf{0.7814} & \textbf{0.8634} \\
\cmidrule{1-10}
\multirow{3}{*}{Epinions}   & 0  & 0.8582 & 0.7123 & 0.8698 & 0.9223 & \textbf{0.8735} & \textbf{0.7305} & \textbf{0.8745} & \textbf{0.9275} \\
                            & 10 & 0.7943 & 0.6579 & \textbf{0.8625} & 0.9123 & \textbf{0.8448} & \textbf{0.7325} & 0.8605          & \textbf{0.9184} \\
                            & 20 & 0.7331 & 0.6293 & \textbf{0.8241} & \textbf{0.8996} & \textbf{0.7882} & \textbf{0.6777} & 0.8215          & 0.8935          \\
\bottomrule
\end{tabular}}
\label{Q5_balance-attack}
\end{table}

\begin{table}[!h] \centering
\caption{Effectiveness of balance augmentation under FlipAttack.}
\resizebox{\linewidth}{!}{
\begin{tabular}{c c cccc cccc}
\toprule
\multirow{2}{*}{\textbf{Dataset}} & \multirow{2}{*}{\textbf{Ptb(\%)}} & \multicolumn{4}{c}{\textbf{random-SGCL}} & \multicolumn{4}{c}{\textbf{BA-SGCL}} \\
\cmidrule(lr){3-6} \cmidrule(lr){7-10}
& & \multicolumn{1}{c}{AUC} & \multicolumn{1}{c}{Macro-F1} & \multicolumn{1}{c}{Micro-F1} & \multicolumn{1}{c}{Binary-F1} & \multicolumn{1}{c}{AUC} & \multicolumn{1}{c}{Macro-F1} & \multicolumn{1}{c}{Micro-F1} & \multicolumn{1}{c}{Binary-F1} \\
\midrule
\multirow{3}{*}{BitcoinAlpha} & 0  & 0.8523 & 0.7329 & 0.9441 & 0.9712 & \textbf{0.8942} & \textbf{0.7778} & \textbf{0.9528} & \textbf{0.9753} \\
                              & 10 & 0.7423 & 0.5889 & 0.8512 & 0.9198 & \textbf{0.7824} & \textbf{0.6248} & \textbf{0.8573} & \textbf{0.9201} \\
                              & 20 & 0.6182 & 0.5021 & 0.7421 & 0.8327 & \textbf{0.6742} & \textbf{0.5392} & \textbf{0.7454} & \textbf{0.8364} \\
\cmidrule{1-10}
\multirow{3}{*}{BitcoinOTC}   & 0  & 0.8829 & 0.7789 & 0.9351 & 0.9612 & \textbf{0.9104} & \textbf{0.8074} & \textbf{0.9385} & \textbf{0.9665} \\
                              & 10 & 0.7179 & 0.6143 & 0.8427 & 0.9024 & \textbf{0.7447} & \textbf{0.6633} & \textbf{0.8577} & \textbf{0.9118} \\
                              & 20 & 0.6988 & 0.5613 & 0.7989 & 0.8769 & \textbf{0.7223} & \textbf{0.6273} & \textbf{0.8077} & \textbf{0.8881} \\
\cmidrule{1-10}
\multirow{3}{*}{Slashdot}     & 0  & 0.8853 & 0.7214 & 0.8679 & 0.9123 & \textbf{0.8953} & \textbf{0.7543} & \textbf{0.8795} & \textbf{0.9304} \\
                              & 10 & 0.7332 & 0.5721 & \textbf{0.8097} & 0.8687 & \textbf{0.7623} & \textbf{0.6098} & 0.8028          & \textbf{0.8733} \\
                              & 20 & 0.6712 & 0.5179 & 0.7527 & 0.8305 & \textbf{0.7032} & \textbf{0.5417} & \textbf{0.7633} & \textbf{0.8383} \\
\cmidrule{1-10}
\multirow{3}{*}{Epinions}     & 0  & 0.8647 & 0.6932 & 0.8726 & 0.9228 & \textbf{0.8735} & \textbf{0.7305} & \textbf{0.8745} & \textbf{0.9275} \\
                              & 10 & 0.7267 & 0.5823 & 0.8198 & \textbf{0.8943} & \textbf{0.7563} & \textbf{0.6014} & \textbf{0.8327} & 0.8901          \\
                              & 20 & 0.6665 & 0.5402 & 0.7994 & 0.8279 & \textbf{0.6946} & \textbf{0.5615} & \textbf{0.8027} & \textbf{0.8436} \\
\bottomrule
\end{tabular}}
\label{Q5_flipattack}
\end{table}

\subsection{Analysis of Balance Augmentation (Q5)}

To evaluate the effectiveness of balance augmentation, we compare BA-SGCL with random-SGCL, a control model where signs in one augmented view are randomly perturbed while the other view remains unchanged. All other components and settings remain identical between the two models. Tables \ref{Q5_balance-attack} and \ref{Q5_flipattack} present detailed comparative results under balance-attack and FlipAttack, respectively. BA-SGCL consistently outperforms random-SGCL, demonstrating the effectiveness of our proposed balance augmentation strategy.

\subsection{Ablation Study}
To validate that our model's enhanced robustness stems from the combination of GCL framework and balance augmentation rather than the SDGCN encoder alone, we conducted experiments replacing SDGCN with alternative encoders such as SGCN \cite{derr2018signed}, while maintaining all other components. Tables \ref{ablation_study_balance-attack} and \ref{ablation_study_flipattack} compare the performance of BA-SGCL using SGCN encoder against the original SGCN model under balance-attack and FlipAttack, respectively. BA-SGCL (SGCN encoder) significantly outperforms the baseline SGCN model, confirming the effectiveness of our proposed GCL framework and balance augmentation strategy.

\begin{table}[!h] \centering
\caption{Ablation study under balance-attack.}
\resizebox{\linewidth}{!}{
\begin{tabular}{c c cccc cccc}
\toprule
\multirow{2}{*}{\textbf{Dataset}} & \multirow{2}{*}{\textbf{Ptb(\%)}} & \multicolumn{4}{c}{\textbf{SGCN}} & \multicolumn{4}{c}{\textbf{BA-SGCL (SGCN encoder)}} \\
\cmidrule(lr){3-6} \cmidrule(lr){7-10}
& & \multicolumn{1}{c}{AUC} & \multicolumn{1}{c}{Macro-F1} & \multicolumn{1}{c}{Micro-F1} & \multicolumn{1}{c}{Binary-F1} & \multicolumn{1}{c}{AUC} & \multicolumn{1}{c}{Macro-F1} & \multicolumn{1}{c}{Micro-F1} & \multicolumn{1}{c}{Binary-F1} \\
\midrule
\multirow{3}{*}{BitcoinAlpha} & 0  & 0.7997 & 0.6652 & 0.8523 & 0.9044 & \textbf{0.8447} & \textbf{0.7226} & \textbf{0.9298} & \textbf{0.9615} \\
                              & 10 & 0.6917 & 0.5128 & 0.6604 & 0.7682 & \textbf{0.7723} & \textbf{0.6113} & \textbf{0.8832} & \textbf{0.9225} \\
                              & 20 & 0.6532 & 0.4708 & 0.6025 & 0.7343 & \textbf{0.6954} & \textbf{0.5383} & \textbf{0.7921} & \textbf{0.8718} \\
\cmidrule{1-10}
\multirow{3}{*}{BitcoinOTC} & 0  & 0.8257 & 0.7505 & 0.8794 & 0.9293 & \textbf{0.8957} & \textbf{0.7821} & \textbf{0.9244} & \textbf{0.9557} \\
                            & 10 & 0.7508 & 0.6147 & 0.7477 & 0.8405 & \textbf{0.8229} & \textbf{0.6743} & \textbf{0.8578} & \textbf{0.9143} \\
                            & 20 & 0.6982 & 0.5625 & 0.6836 & 0.7922 & \textbf{0.7449} & \textbf{0.5823} & \textbf{0.7742} & \textbf{0.8598} \\
\cmidrule{1-10}
\multirow{3}{*}{Slashdot}   & 0  & 0.8156 & 0.6838 & 0.8124 & 0.8785 & \textbf{0.8851} & \textbf{0.7044} & \textbf{0.8724} & \textbf{0.9211} \\
                            & 10 & 0.6897 & 0.5714 & 0.6515 & 0.7525 & \textbf{0.7559} & \textbf{0.5962} & \textbf{0.7973} & \textbf{0.8721} \\
                            & 20 & 0.6344 & 0.5202 & 0.5935 & 0.7025 & \textbf{0.6776} & \textbf{0.5543} & \textbf{0.7685} & \textbf{0.8427} \\
\cmidrule{1-10}
\multirow{3}{*}{Epinions}   & 0  & 0.7767 & 0.6952 & 0.8184 & 0.8866 & \textbf{0.8623} & \textbf{0.7083} & \textbf{0.8661} & \textbf{0.9292} \\
                            & 10 & 0.7387 & 0.6122 & 0.7035 & 0.7973 & \textbf{0.7998} & \textbf{0.6737} & \textbf{0.8574} & \textbf{0.9046} \\
                            & 20 & 0.6885 & 0.5602 & 0.6456 & 0.7494 & \textbf{0.7478} & \textbf{0.6122} & \textbf{0.7992} & \textbf{0.8575} \\
\bottomrule
\end{tabular}}
\label{ablation_study_balance-attack}
\end{table}

\begin{table}[!h] \centering
\caption{Ablation study under FlipAttack.}
\resizebox{\linewidth}{!}{
\begin{tabular}{c c cccc cccc}
\toprule
\multirow{2}{*}{\textbf{Dataset}} & \multirow{2}{*}{\textbf{Ptb(\%)}} & \multicolumn{4}{c}{\textbf{SGCN}} & \multicolumn{4}{c}{\textbf{BA-SGCL (SGCN encoder)}} \\
\cmidrule(lr){3-6} \cmidrule(lr){7-10}
& & \multicolumn{1}{c}{AUC} & \multicolumn{1}{c}{Macro-F1} & \multicolumn{1}{c}{Micro-F1} & \multicolumn{1}{c}{Binary-F1} & \multicolumn{1}{c}{AUC} & \multicolumn{1}{c}{Macro-F1} & \multicolumn{1}{c}{Micro-F1} & \multicolumn{1}{c}{Binary-F1} \\
\midrule
\multirow{3}{*}{BitcoinAlpha} & 0  & 0.7997 & 0.6652 & 0.8523 & 0.9044 & \textbf{0.8447} & \textbf{0.7226} & \textbf{0.9298} & \textbf{0.9615} \\
                              & 10 & 0.6743 & 0.5547 & 0.6742 & 0.7987 & \textbf{0.7445} & \textbf{0.5689} & \textbf{0.8332} & \textbf{0.9098} \\
                              & 20 & 0.5607 & 0.4327 & 0.5794 & 0.7078 & \textbf{0.6052} & \textbf{0.4925} & \textbf{0.7121} & \textbf{0.8232} \\
\cmidrule{1-10}
\multirow{3}{*}{BitcoinOTC} & 0  & 0.8257 & 0.7505 & 0.8794 & 0.9293 & \textbf{0.8957} & \textbf{0.7821} & \textbf{0.9244} & \textbf{0.9557} \\
                            & 10 & 0.6848 & 0.6083 & 0.7234 & 0.8257 & \textbf{0.6929} & \textbf{0.6077} & \textbf{0.8112} & \textbf{0.8923} \\
                            & 20 & 0.6312 & 0.5517 & 0.6421 & 0.7593 & \textbf{0.6639} & \textbf{0.5723} & \textbf{0.7661} & \textbf{0.8524} \\
\cmidrule{1-10}
\multirow{3}{*}{Slashdot}   & 0  & 0.8156 & 0.6838 & 0.8124 & 0.8785 & \textbf{0.8851} & \textbf{0.7044} & \textbf{0.8724} & \textbf{0.9211} \\
                            & 10 & 0.6415 & 0.5217 & 0.6863 & 0.7471 & \textbf{0.7278} & \textbf{0.5752} & \textbf{0.8003} & \textbf{0.8661} \\
                            & 20 & 0.6045 & 0.4773 & 0.6263 & 0.7096 & \textbf{0.6723} & \textbf{0.5114} & \textbf{0.7476} & \textbf{0.8212} \\
\cmidrule{1-10}
\multirow{3}{*}{Epinions}   & 0  & 0.7767 & 0.6952 & 0.8184 & 0.8866 & \textbf{0.8623} & \textbf{0.7083} & \textbf{0.8661} & \textbf{0.9292} \\
                            & 10 & 0.6894 & 0.5212 & 0.7094 & 0.7631 & \textbf{0.7259} & \textbf{0.5778} & \textbf{0.8132} & \textbf{0.8776} \\
                            & 20 & 0.6144 & 0.4888 & 0.6831 & 0.7284 & \textbf{0.6662} & \textbf{0.5387} & \textbf{0.7889} & \textbf{0.8236} \\
\bottomrule
\end{tabular}}
\label{ablation_study_flipattack}
\end{table}

\subsection{Parameter Analysis}
We analyze the sensitivity of BA-SGCL to the hyperparameter $\alpha$, which balances the contrastive and label losses. We vary $\alpha$ across a wide range from $10^{-3}$ to $10^3$ and evaluate performance under $10\%$ and $20\%$ perturbation rates. As shown in Fig.~\ref{parameter analysis 0.1} and Fig.~\ref{parameter analysis 0.2}, the model's performance is sensitive to this parameter. While an appropriate value for $\alpha$ leads to strong robustness, extremely high or low values degrade performance. This indicates that an optimal balance between the self-supervised contrastive task and the supervised prediction task is crucial for achieving the best results.

\begin{figure}[!t] \centering
\centerline{\includegraphics[width=0.45\textwidth]{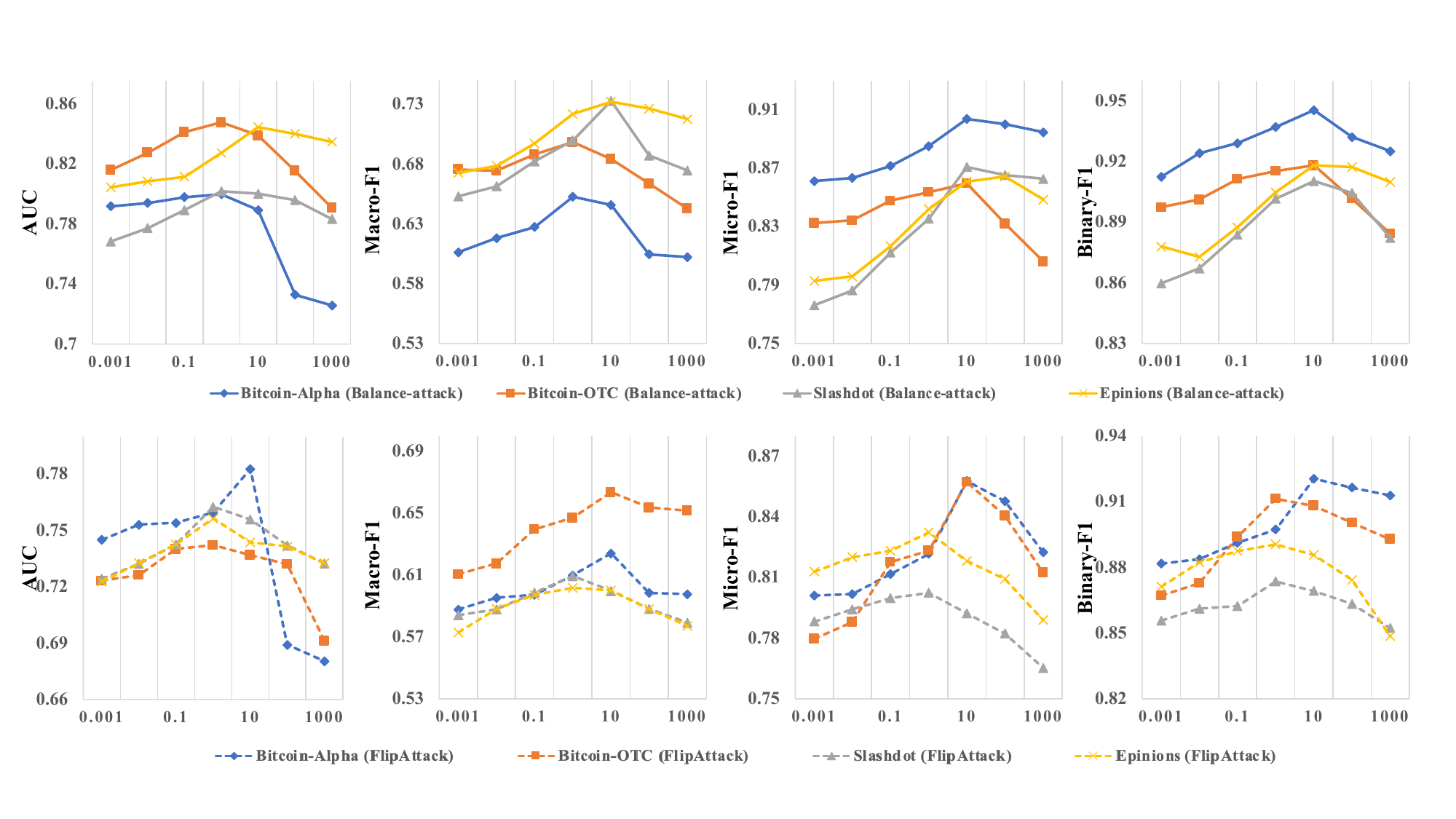}}
\caption{Parameter analysis under perturbation magnitude of $10\%$. Results for balance-attack (top) and FlipAttack (bottom) are shown.}
\label{parameter analysis 0.1}
\end{figure}

\begin{figure}[!h] \centering
\centerline{\includegraphics[width=0.45\textwidth]{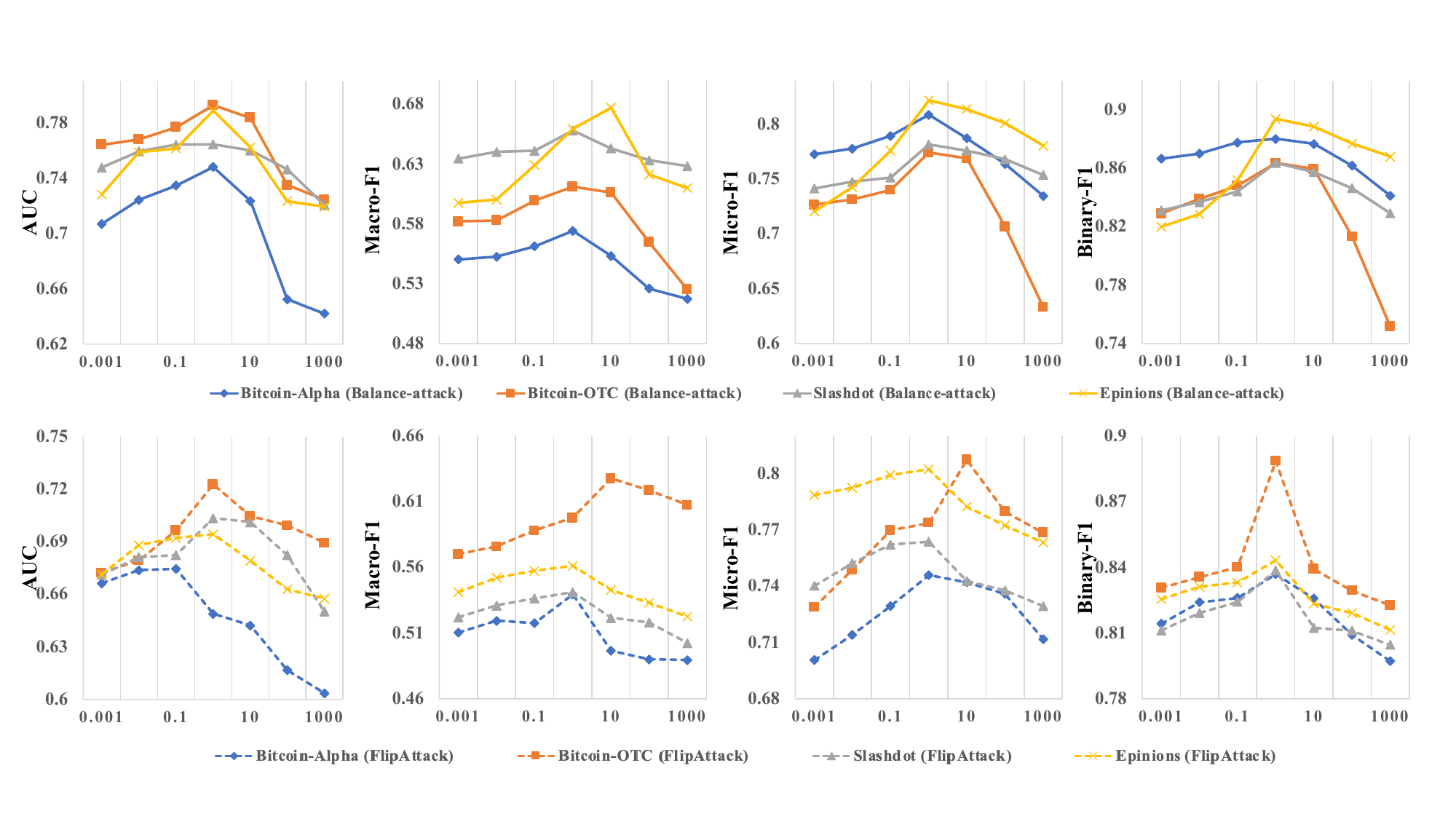}}
\caption{Parameter analysis under perturbation magnitude of $20\%$. Results for balance-attack (top) and FlipAttack (bottom) are shown.}
\label{parameter analysis 0.2}
\end{figure}

\section{Conclusion}
In this work, we demonstrate how balance theory's fundamental role in SGNNs introduces inherent vulnerabilities to adversarial attacks. We propose balance-attack, an efficient strategy targeting graph balance degree, and identify the ``Irreversibility of Balance-related Information" phenomenon in existing defense mechanisms. To address this, we develop BA-SGCL, which leverages contrastive learning with balance augmentation to maintain robust graph representations. Extensive experiments validate both the effectiveness of our attack and the enhanced resilience provided by BA-SGCL, advancing the security of signed graph analysis. This work pioneers robust learning against adversarial attacks in signed graph representation learning, establishing a foundation for future theoretical and empirical research.

\newpage

\vfill

\end{document}